\documentclass[twoside,11pt]{article}

 \usepackage[abbrvbib, preprint]{jmlr2e}
 \usepackage{algorithm}
\usepackage{algorithmic}
\usepackage{amsmath} 
\usepackage{subfig}
\usepackage[nameinlink,capitalize,noabbrev]{cleveref}
\usepackage{amssymb}
\usepackage{mathtools}
\usepackage{enumitem}
\usepackage{bm}
\usepackage{xspace}
\usepackage{xcolor}
\usepackage{hyperref}
\hypersetup{
	colorlinks,
    linkcolor=black,
	urlcolor=pink,
	citecolor=blue}

\DeclareMathOperator*{\argmin}{argmin}

\newcommand{\grad}{\nabla}

\newcommand{\cS}{\mathcal{S}}
\newcommand{\cN}{\mathcal{N}}
\newcommand{\cD}{\mathcal{D}}
\newcommand{\cL}{\mathcal{L}}
\newcommand{\cW}{\mathcal{W}}

\newcommand{\cP}{\mathcal{P}}
\newcommand{\cZ}{\mathcal{Z}}
\newcommand{\PR}{\mathrm{PR}}
\newcommand{\DPR}{\mathrm{DPR}}
\newcommand{\PRhat}{\widehat{\PR}}
\newcommand{\DPRhat}{\widehat{\DPR}}
\newcommand{\Dhat}{\widehat{\cD}}

\newcommand{\thetaPO}{\theta_{\mathrm{PO}}}
\newcommand{\thetaPS}{\theta_{\mathrm{PS}}}

\newcommand{\Eclean}{E_{\mathrm{clean},\delta}}
\newcommand{\ndeploy}{n^{\mathrm{deploy}}}
\newcommand{\nopen}{n^{\mathrm{open}}}
\newcommand{\ntotal}{n_{\mathrm{total}}}
\newcommand{\ndistinct}{n_{\mathrm{distinct}}}
\newcommand{\hmax}{h_{\max}}
\newcommand{\pmax}{p_{\max}}
\newcommand{\Dtheta}{D}
\newcommand{\bbE}{\mathbb{E}}
\newcommand{\bbP}{\mathbb{P}}
\newcommand{\StroquOOL}{\normalfont \texttt{StroquOOL}\xspace}
\newcommand{\SequOOL}{\normalfont{\texttt{SequOOL}}\xspace}

\newcommand{\StoSOO}{\texttt{StoSOO}\xspace}
\newcommand{\POO}{\texttt{POO}\xspace}
\newcommand{\DOO}{\texttt{DOO}\xspace}
\newcommand{\SOO}{\texttt{SOO}\xspace}
\newcommand{\DOOP}{\texttt{DOOP}\xspace}
\newcommand{\SOOP}{\texttt{SOOP}\xspace}
\newcommand{\Zooming}{\texttt{Zooming}\xspace}
\newcommand{\SZooming}{\texttt{SZooming}\xspace}
\newcommand{\GPO}{\texttt{GPO}\xspace}
\newcommand{\HCT}{\texttt{HCT}\xspace}
\newcommand{\PCT}{\texttt{PCT}\xspace}
\newcommand{\VHCT}{\texttt{VHCT}\xspace}
\newcommand{\HOO}{\texttt{HOO}\xspace}
\newcommand{\Direct}{\texttt{DiRect}\xspace}

\newtheorem{assumption}{Assumption}

\usepackage{lastpage}
\jmlrheading{23}{2022}{1-\pageref{LastPage}}{1/21; Revised 5/22}{9/22}{21-0000}{Author One and Author Two}

\ShortHeadings{Parameter-Free Performative Regret Minimization}{Parameter-Free Performative Regret Minimization}
\firstpageno{1}

\begin{document}

\title{Parameter-Free Algorithms for Performative Regret Minimization under Decision-Dependent Distributions}

\author{\name Sungwoo Park$^1$  \email s.park@kaist.ac.kr 
       \AND
        \name Junyeop Kwon$^1$ \email junyeopk@kaist.ac.kr
       \AND
       \name Byeongnoh Kim$^2$ \email b-n.kim@samsung.com 
       \AND
       \name Suhyun Chae$^2$ \email suhyun.chae@samsung.com 
       \AND
       \name Jeeyong Lee$^2$ \email jiyong.lee@samsung.com 
       \AND
       \name Dabeen Lee$^{1,\dagger}$ \email dabeenl@kaist.ac.kr
       \AND
       \addr 
        $^1$Department of Industrial and Systems Engineering, KAIST, Daejeon 34141, South Korea\\
        $^2$Device Solutions Research, Samsung Electronics, Hwaseong, Gyeonggi 18448, South Korea\\
        $^\dagger$ corresponding author
       }

\editor{My editor}

\maketitle

\begin{abstract}%
This paper studies performative risk minimization, a formulation of stochastic optimization under decision-dependent distributions. We consider the general case where the performative risk can be non-convex, for which we develop efficient parameter-free optimistic optimization-based methods. Our algorithms significantly improve upon the existing Lipschitz bandit-based method in many aspects. In particular, our framework does not require knowledge about the sensitivity parameter of the distribution map and the Lipshitz constant of the loss function. This makes our framework practically favorable, together with the efficient optimistic optimization-based tree-search mechanism. We provide experimental results that demonstrate the numerical superiority of our algorithms over the existing method and other black-box optimistic optimization methods.
\end{abstract}


\section{Introduction}

In the realm of stochastic optimization, where navigating uncertainty is paramount, distributional shifts stand out as a significant challenge. Among the various sources of these shifts, one particularly intriguing phenomenon stems from feedback mechanisms intricately linked to decision-making processes. This feedback loop alters the distribution that governs the stochastic environment of the system, creating a dynamic landscape where decisions shape and are shaped by distributions. For example, the decisions made by a dynamic resource allocation algorithm for a renewable energy grid not only influence the immediate allocation of resources but also affect the underlying distribution of factors like energy demand and supply. Classifiers, such as insurance underwriting systems, often promote a shift in behavior within the population to improve their labels. Predictions of stock prices wield significant influence over trading decisions. Moreover, election predictions have the potential to shape and influence voter behavior, which in turn can impact voting results.

Decision-making processes under such phenomena can be formulated as stochastic optimization under \emph{decision-dependent distributions}. \cite{perdomo20a} proposed the notion of the \emph{distribution map} to consider decision-dependent distributions for stochastic optimization models. That is, the distribution $\cD(\theta)$ of the parameter $z$ capturing the stochastic environment depends on the decision $\theta$. Here, $\theta$ may encode the resource allocation decision for a renewable energy grid and the election prediction, in which case $z$ corresponds to the energy demand and the voting results, respectively. For machine learning, we can associate $\theta$ with predictive models and $z$ with data. Then the objective is to minimize the \emph{performative risk} under a loss function $f$, defined as 
$$\PR(\theta):= \mathbb{E}_{z\sim \cD(\theta)}\left[f(\theta, z)\right].$$
The expression \emph{performative} comes from the term \emph{performative prediction}~\citep{perdomo20a}, which implies the phenomenon where predictions influence the outcomes. The goal of this paper is to design an efficient algorithmic framework for minimizing the performative risk which models stochastic optimization under decision-dependent distributions.

\subsection{Existing Methods for Performative Risk Minimization}

Unlike the standard stochastic optimization problem, the decision $\theta$ may affect the underlying distribution $\cD(\theta)$. Hence, a natural starting point to minimize the performative risk is to consider the following iterative algorithm, referred to as \emph{repeated risk minimization} (RRM). Given an initial solution $\theta_0\in\Theta$ where $\Theta$ is the domain, we apply
\begin{equation}\label{RRM}
\theta_{t+1}\in \argmin_{\theta\in \Theta} \mathbb{E}_{z\sim \cD(\theta_{t})}\left[f(\theta, z)\right]\tag{RRM}
\end{equation}
for $t\geq 0$. Here, computing the next iterate $\theta_{t+1}$ requires solving a stochastic optimization instance where the underlying distribution is fixed with $\cD(\theta_t)$. Another approach is a gradient-based method such as
\begin{equation}\label{RGD}
\theta_{t+1} = \theta_t - \eta_t\mathbb{E}_{z\sim \cD(\theta_{t})}\left[\grad f(\theta_t, z)\right]\tag{RGD}
\end{equation}
where $\eta_t$ is a step size, and we refer to this procedure as \emph{repeated gradient descent} (RGD). Note that running \ref{RGD}, as well as \ref{RRM}, is based on access to the distribution $\cD(\theta_t)$ for every iteration $t\geq 0$, which may not be feasible in practice. A more sample-efficient method is to apply the standard \emph{stochastic gradient descent} (SGD) update, given by 
\begin{equation}\label{SGD}
\theta_{t+1} = \theta_t - \eta_t \grad f(\theta_t, z_t)\quad \text{where}\quad z_t\sim \cD(\theta_t).\tag{SGD}
\end{equation}
\cite{Drusvyatskiy} analyzed variants of~\ref{SGD} such as stochastic proximal gradient, proximal point, clipped gradient, and accelerated gradient methods.

Convergence of these iterative methods has been established; \citep{perdomo20a} for RRM, \citep{perdomo20a,perdomo20b} for RGD, and \citep{perdomo20b} for SGD. They showed convergence to a \emph{performatively stable} solution, under some strong convexity and smoothness assumptions on the loss function $f$. Here, we say that a solution $\thetaPS$ is performatively stable if it satisfies
$$\thetaPS \in \argmin_{\theta\in \Theta} \mathbb{E}_{z\sim \cD(\thetaPS)}\left[f(\theta, z)\right]$$
In particular, $\thetaPS$ is a fixed point of \ref{RRM}. However, let alone the validity of the structural assumptions on the loss functions, the performatively stable solution $\thetaPS$ is in general not a minimizer of the performative risk~\citep{perdomo20a, miller21a}. Let $\thetaPO$ denote a minimizer of the performative risk, i.e.,
$$\thetaPO \in \argmin_{\theta\in \Theta} \mathbb{E}_{z\sim \cD(\theta)}\left[f(\theta, z)\right].$$
It turns out that $\PR(\thetaPS)$ can be arbitrarily large compared to  $\PR(\thetaPO)$~\citep{miller21a}. 

Derivative-free zeroth-order optimization methods have been proposed to minimize the performative risk directly \citep{izzo21a,miller21a,izzo22a,Ray22}. For the derivative-free methods to minimize the performative risk $\PR(\theta)$, the requirement, however, is that $\PR(\theta)$ is a convex function of the decision $\theta$. \cite{miller21a} provided some sufficient conditions on the distribution map to guarantee the convexity of the performative risk. They argued that if the distribution map satisfies a certain stochastic dominance condition, which is related to \emph{stochastic orders}~\citep{shaked2007stochastic}, then convexity of the loss function $f$ leads to a convex performative risk. Nevertheless, as noted by~\cite{perdomo20a}, the performative risk is non-convex in general even if the loss function is convex. 

To tackle the general case of non-convex performative risk, \cite{jagadeesan22a} developed a bandit optimization-based algorithm. The problem of minimizing the performative risk is indeed a bandit optimization problem because until the decision $\theta$ is deployed it is difficult to estimate the distribution $\cD(\theta)$ and thus the performative risk $\PR(\theta)$. That said, the framework of \cite{jagadeesan22a} is inspired by the zooming algorithm (\Zooming) for Lipschitz bandits due to \cite{Kleinberg08}. The core idea is to adaptively discretize the solution space $\Theta$ thereby narrowing down the location of the optimal decision $\thetaPO$. In fact, if the loss function $f$ is $L_\theta$-Lipschitz continuous in $\theta$ and $L_z$-Lipschitz continuous in $z$, the \emph{$\varepsilon$-sensitivity} of the distribution map (defined formally in \Cref{sec:setting}) implies that $\PR(\theta)$ is $(L_\theta+L_z\varepsilon)$-Lipschitz continuous~\citep{jagadeesan22a}. Here, the $\varepsilon$-sensitivity measures how much the distributions $\cD(\theta)$ and $\cD(\theta')$ can differ for two distinct decisions $\theta$ and $\theta'$. Then applying \Zooming directly on $\PR(\theta)$ would guarantee a sublinear regret. 

Although this lays down a good starting point, direct application of \Zooming fails to utilize the fact that the feedback obtained after deploying decision $\theta$ is $\cD(\theta)$, based on which the learner can evaluate $\PR(\theta)$ but also infer the distribution $\cD(\theta')$ of other solutions $\theta'$ using the $\varepsilon$-sensitivity. \cite{jagadeesan22a} referred to this as \emph{performative feedback}. Building on this idea, they developed a variant of \Zooming, and they provided a regret upper bound that is parameterized by not $L_\theta+L_z\varepsilon$ but $L_z\varepsilon$. Here, note that $L_z\varepsilon$ vanishes as $\varepsilon\to 0$ while $L_\theta+L_z\varepsilon$ does not. Moreover, $L_z\varepsilon$ does not depend on $L_\theta$, so the algorithm works even when $L_\theta$ is not bounded. The two main components of their algorithm are adaptive discretization and sequential elimination based on \emph{performative confidence bounds} which we explain in \Cref{sec:setting}. 

The algorithm of \cite{jagadeesan22a} solves performative risk minimization, but several issues hinder its practical implementation. First, to implement the adaptive discretization procedure, we need to know the Rademacher complexity $\mathfrak{C}^*(f)$ of learning the performative risk $\PR(\theta)$ under the loss function $f$ based on data samples from distribution $\cD(\theta)$. The Rademacher complexity parameter $\mathfrak{C}^*(f)$ can be very high depending on the structure of $f$. Second, to build a performative confidence bound, we need the Lipschitz constant $L_z$ and the sensitivity parameter $\varepsilon$. One may argue that there is a way of estimating the Lipschitz constant $L_z$ for a known class of loss functions, but the sensitivity parameter $\varepsilon$ determines the global landscape of the distribution map, which means that it would be difficult to measure $\varepsilon$ in advance. Third, one iteration of the algorithm is computationally expensive. This is because each time a decision $\theta$ is deployed, we need to compute the performative confidence bound for every solution $\theta'\in\Theta$ remaining in the search space. Such an issue is inherent in Lipschitz bandit-based methods. Although \cite{jagadeesan22a} did not demonstrate an implementation of their algorithm, our numerical results in \Cref{sec:numerical} show that the algorithm is not efficient and incurs a high regret in practice.

The aforementioned limitations of the existing method due to~\cite{jagadeesan22a} for performative risk minimization motivate the following question.
\begin{quote}
\emph{Can we design a practical algorithm for performative risk minimization that relies on minimal knowledge about the problem parameters?}
\end{quote}
In this paper, we devise efficient parameter-free algorithms for performative risk minimization. We not only demonstrate strong theoretical performance guarantees but also show experimental results to highlight their numerical effectiveness. 

\subsection{Our Contributions}

As in~\citep{jagadeesan22a}, we study the problem of minimizing the performative risk with performative feedback. The algorithm of~\cite{jagadeesan22a} is an adaptation of \Zooming by~\cite{Kleinberg08}, and as a result, it requires knowledge of problem parameters such as the Rademacher complexity $\mathfrak{C}^*(f)$, the Lipschitz constant $L_z$, and the sensitivity parameter $\varepsilon$.

Our main contribution is to design practical algorithms that do not assume knowledge of the problem parameters. To develop such parameter-free algorithms, we build upon the idea of \emph{optimistic optimization} methods that may adapt to unknown smoothness of the objective function. Here, parameter-free optimistic optimization methods originate from the simultaneous optimistic optimization (\SOO) algorithm~\citep{munos-soo} and the stochastic extension of \SOO (\StoSOO) algorithm~\citep{stosoo}, and they are devised to optimize black-box objective functions that are possibly non-convex.

Our algorithms are inspired by two more recent optimistic optimization-based parameter-free methods due to~\cite{bartlett19a}, \SequOOL for the deterministic evaluation case and \StroquOOL for the noisy case. We start by considering the conceptual setting where the distribution $\cD(\theta)$ associated with the deployed decision $\theta$ can be fully observed. We call this case the \emph{full-feedback} setting. Next, we study the more practically relevant setting where we obtain a few samples from $\cD(\theta)$ after deploying decision $\theta$, and we refer to this case as the \emph{data-driven} setting. We develop our algorithms for the full-feedback setting and the data-driven setting based on \SequOOL and \StroquOOL, respectively.

To highlight our results early in this paper, let us provide an informal summary of our main theorems. The following states a performance guarantee for the full-feedback case. 
\begin{theorem}[Full-Feedback Case, Informal]\label{thm-informal:full-feedback}
Suppose that the distribution map satisfies the $\varepsilon$-sensitivity condition and the loss function $f(\theta,z)$ is $L_z$-Lipschitz continuous in $z$ for any $\theta$. Let $d$ denote the $L_z\varepsilon$-near-optimality dimension. For the full-feedback setting, \Cref{algorithm0} after $T$ decision deployments for a sufficiently large $T$ finds a solution $\theta$ with
$$\PR(\theta) - \PR(\thetaPO) =\begin{cases}
{L_z\varepsilon}\cdot {2^{O\left(-\frac{T}{\log T}\right)}},& \text{if $d=0$},\\
\tilde O\left(L_z \varepsilon \cdot T^{-\frac{1}{d}}\right), &\text{if $d>0$}.
\end{cases}$$
\end{theorem}
Here, the optimality gap bounds hide dependence on the ambient dimension $D$ of the decision domain $\Theta$. The notion of \emph{near-optimality dimension} was first introduced by \cite{x-armed-bandits}, and they argued that the near-optimality dimension and the \emph{zooming dimension} due to \cite{Kleinberg08} are closely related. In this paper, we use a more refined definition of the near-optimality dimension due to \cite{Grill15}. 

Note that the optimality gap bounds as well as the near-optimality dimension depend on parameters $L_z$ and $\varepsilon$ but not on $L_\theta$. In fact, direct application of \SequOOL to minimize the performative risk $\PR(\theta)$ would result in dependence on $L_\theta + L_z \varepsilon$ as the Lipschitz constant of $\PR(\theta)$ is $L_\theta + L_z \varepsilon$~\citep{jagadeesan22a}. More precisely, we would need the $(L_\theta + L_z\varepsilon)$-near-optimality dimension, and the resulting bounds would be $(L_\theta+L_z\varepsilon)2^{O(-T/\log T)}$ and $\tilde O((L_\theta+L_z\varepsilon) T^{-1/d})$. The important distinction is that our optimality gap bounds vanish as $\varepsilon$ becomes arbitrarily small, which setting corresponds to the standard stochastic optimization problem with decision-agnostic distributions. Moreover, the $L_z\varepsilon$-near-optimality dimension is always less than or equal to the $(L_\theta + L_z\varepsilon)$-near-optimality dimension. Another aspect to highlight in \Cref{thm:full-feedback} is that when $d=0$, we show that the optimality gap decays at an exponentially fast rate, which was not discovered by~\cite{jagadeesan22a}.

Next, we state our performance guarantee for the data-driven setting where we receive a finite number of samples from $\cD(\theta)$ after deploying decision $\theta$. The optimality gap bounds on \Cref{algorithm1} hide dependence on the dimension $D$, the number of samples received after each decision deployment, the Rademacher complexity of learning the performative risk. 
\begin{theorem}[Data-Driven Case, Informal]\label{thm-informal:data-driven}
Assume the same conditions on the distribution map and the loss function. Let $d$ denote the $L_z\varepsilon$-near-optimality dimension. For the data-driven setting, \Cref{algorithm1} after $T$ decision deployments for a sufficiently large $T$ finds a solution $\theta$  such that with high probability,
\begin{align*}
\PR(\theta) - \PR(\thetaPO)
&=\begin{cases}
{L_z\varepsilon}\cdot {2^{O\left(-\frac{T}{\log T}\right)}},& \text{low-noise regime with $d=0$},\\
\tilde O\left(L_z \varepsilon\cdot T^{-\frac{1}{d}}\right), &\text{low-noise regime with $d>0$},\\
\tilde O\left({T}^{-\frac{1}{2}} + (L_z \varepsilon)^{\frac{d}{d+2}} \cdot T^{-\frac{1}{d+2}}\right),&\text{high-noise regime.}
\end{cases}
\end{align*}
\end{theorem}
The \emph{low-noise} and \emph{high-noise} regimes are defined in \Cref{sec:data-driven}. In particular, for the high-noise regime, the bound is $\tilde O(T^{-1/2} + L_z \varepsilon T^{-{1}/{d}})$ which incurs the additional term $T^{-1/2}$ due to errors in estimating the performative risk through noisy feedback. In particular, the case of $\varepsilon=0$ is under the high-noise regime, in which case the bound reduces to $\tilde O(T^{-1/2})$.

Lastly, we test the numerical performance of our framework on instances in which the associated performative risk is non-convex. The experimental results show that our algorithms outperform the existing methods that include not only the sequential zooming algorithm of~\cite{jagadeesan22a} but also \SOO~\citep{munos-soo}, \StoSOO~\citep{stosoo}, \SequOOL, and \StroquOOL~\citep{bartlett19a} applied directly to the performative risk as a black-box function without utilizing the performative feedback.

\section{Related Work}

This section summarizes prior work on performative prediction and optimistic optimization.

\subsection{Performative Prediction and Performative Risk Minimization}

Previous work on performative prediction has mainly focused on first-order and zeroth-order gradient-based optimization methods~\citep{perdomo20a,perdomo20b,Drusvyatskiy,brown22a,miller21a,izzo21a,maheshwari22a, li22c,Ray22,dong23b,izzo22a}. Convergence of these gradient-based methods to performatively stable solutions is studied, and \cite{miller21a,izzo21a} discovered some convexity conditions under which some gradient-based methods converge to a performative optimal solution. Although a performatively stable solution provides a good proxy for a performatively optimal solution, its performance can be arbitrarily worse than the optimum. Moreover, in general, the performative risk is non-convex and does not satisfy the convexity conditions. For the general case, \cite{jagadeesan22a} developed a variant of \Zooming for minimizing the performative risk. \cite{mofakhami23a} studied \ref{RRM} for training non-convex neural networks, but they considered a different setting in terms of defining the $\varepsilon$-sensitivity of the distribution map. For a comprehensive survey on performative prediction, we refer the reader to \cite{hardt2023performative} and references therein.

One of the most closely related application domains is \emph{strategic classification}~\citep{Dalvi,brueckner12a,hardt-strategic}, which models a game between an institution deploying a classifier and an agent who adapts its features to increase its likelihood of being positively labeled. Recent work in this area includes~\citep{10.1145/3219166.3219193, NEURIPS2020_ae87a54e,Milli, bechavod21a,NEURIPS2021_812214fb}.

\subsection{Optimistic Optimization}

\emph{Black-box optimization} and \emph{continuum-armed bandits} aim to optimize an objective function under minimal knowledge about the function. Some early work provides algorithms that assume some weak or local smoothness conditions around a global optimal solution, such as~\Zooming~\citep{Kleinberg08}, \HOO~\citep{x-armed-bandits}, \DOO~\citep{munos-soo},~\HCT~\citep{azar14}. Here, \HOO, \DOO, and \HCT are optimistic optimization-based methods, which means that these algorithms use some optimistic estimates of the black-box objective function when running a global search of the solution space. However, \Zooming, \HOO, \DOO, and \HCT require the knowledge of the local smoothness parameter. Then \cite{munos-soo} presented \SOO that works even when the local smoothness parameter is unknown. \cite{stosoo} developed \StoSOO which extends \SOO for the case of stochastic function evaluation, but its convergence guarantee holds for the limited case of the near-optimality dimension being 0. \POO due to \cite{Grill15} and \GPO, \PCT developed by~\cite{xuedong19a} work for more general families of objective functions. Later, \cite{bartlett19a} presented \SequOOL for the deterministic function evaluation case and \StroquOOL for the stochastic case, which work for general families of functions and exhibit state-of-the-art numerical performance. 

Recently, \cite{li2023optimumstatistical} provided \VHCT, which does not require the budget on the number of decision deployments beforehand but needs the knowledge of the smoothness parameter. There exist more algorithms that work under more specific assumptions on smoothness. For example, \Direct~\citep{jones93} and methods for continuum-armed bandits due to~\citep{NIPS2011_7634ea65,10.1007/978-3-642-24412-4_14,pmlr-v70-malherbe17a} can take Lipschitz-continuous objective functions.

\section{Preliminaries: Optimization with Performative Feedback}\label{sec:setting}

In this section, we introduce the basics of performative prediction. Then we explain how to make use of performative feedback for performative risk minimization as established by~\cite{jagadeesan22a}. In addition, we elaborate briefly on some limitations of the performative confidence bound-based zooming algorithm by~\cite{jagadeesan22a}.

As mentioned in the introduction, the $\varepsilon$-sensitivity measures how much the distribution $\cD(\theta)$ can change with changes in decision $\theta$. Formally, we assume that the distribution map satisfies the following. Recall that $\Theta$ denotes the decision domain.
\begin{assumption}[$\varepsilon$-sensitivity]\label{assumption:wasserstein}
A distribution map $\cD(\cdot)$ is \emph{$\varepsilon$-sensitive} with $\alpha>0$ if for any $\theta,\theta'\in \Theta$  we have
$$\cW(\cD(\theta),\cD(\theta'))\leq \varepsilon\|\theta-\theta'\|^\alpha,$$ where $\cW$ denotes the 1-Wasserstein distance.
\end{assumption}
The original definition due to~\cite{perdomo20a} considers the case $\alpha=1$, while our framework allows arbitrary positive values of $\alpha$. We remark that our framework is parameter-free in that we do not require knowledge of the parameters $\varepsilon$ and $\alpha$ in advance. In theory, as we build upon optimistic optimization methods, we may take any \emph{semi-metric} $\ell$, satisfying $\ell(\theta,\theta')=\ell(\theta',\theta)$ and $\ell(\theta,\theta')=0$ if and only if $\theta=\theta'$ for $\theta,\theta'\in\Theta$. That being said, we may run our algorithms regardless of the sensitivity structure of the distribution map, but we derive the theoretical performance guarantees based on the sensitivity structure given in \Cref{assumption:wasserstein}.

Next, we define the notion of performative feedback used to infer the distribution $\cD(\theta)$ as well as the performative risk $\PR(\theta)$ after deploying decision $\theta$.
\begin{assumption}[performative feedback]\label{assumption:performative}
Deploying decision $\theta$ once, we receive feedback about the distribution as follows.
\begin{itemize}
    \item (Full-Feedback Setting) distribution $\cD(\theta)$ itself.
    \item (Data-Driven Setting)  $m_0$ i.i.d. samples $z_\theta^{(1)},\ldots, z_\theta^{(m_0)}$ from distribution $\cD(\theta)$.
\end{itemize}
\end{assumption}
For the data-driven setting, we may deploy the same decision $\theta$ multiple times, say $n$. Then we may construct the empirical distribution $\Dhat(\theta)$ with $nm_0$ i.i.d. samples from $\cD(\theta)$. Using the performative feedback for $\theta$, which provides $\cD(\theta)$ or $\Dhat(\theta)$, we may compute the performative risk $\PR(\theta)$ or its empirical estimate
$$\PRhat(\theta)=\mathbb{E}_{z\sim \Dhat(\theta)}\left[f(\theta, z)\right].$$
Moreover, based on the performative feedback for $\theta$, we may infer the performative risk of other decisions $\theta'$. To be specific, we use the notion of \emph{decoupled performative risk}~\citep{perdomo20a} defined as follows. 
$$\DPR(\theta,\theta')=\mathbb{E}_{z\sim \cD(\theta)}\left[f(\theta', z)\right]\quad\text{and}\quad\DPRhat(\theta,\theta')=\mathbb{E}_{z\sim \Dhat(\theta)}\left[f(\theta', z)\right]$$
for any $\theta,\theta'\in\Theta$ where $\DPR(\theta,\theta')$ is the decoupled performative risk of decision $\theta'$ under distribution $\cD(\theta)$ and $\DPRhat(\theta,\theta')$ is its empirical estimate. The decoupled performative risk offers a good approximation of the performative risk, which we elaborate on below. 

\begin{assumption}\label{assumption:lipschitz}
There is some $L_z>0$ such that $f(\theta,\cdot)$ for any fixed $\theta\in\Theta$ is $L_z$-Lipschitz continuous.
\end{assumption}
Under Assumptions~\ref{assumption:wasserstein} and~\ref{assumption:lipschitz}, the Kantorovich-Rubinstein duality theorem~\citep{KR:58,villani2008optimal} implies the following statement~\citep{jagadeesan22a}.
\begin{lemma}\label{lemma:lipschitz}
Under Assumptions~\ref{assumption:wasserstein} and \ref{assumption:lipschitz}, for $\theta,\theta'\in\Theta$,
$$\left|\PR(\theta') - \DPR(\theta,\theta')\right|\leq L_z\varepsilon\|\theta-\theta'\|^\alpha.$$ 
\end{lemma}
Therefore, as long as decisions $\theta$ and $\theta'$ are close, the decoupled performative risk $\DPR(\theta,\theta')$ deduced based on the performative feedback $\cD(\theta)$ for $\theta$ would be a good proxy for the performative risk $\PR(\theta')$ of decision $\theta'$. By Lemma~\ref{lemma:lipschitz},
$$ \DPR(\theta,\theta') - L_z\varepsilon\|\theta-\theta'\|^\alpha\leq \PR(\theta')\leq \DPR(\theta,\theta') +L_z\varepsilon\|\theta-\theta'\|^\alpha$$
is a valid confidence interval for the performative risk of $\theta'\in\Theta$. 
Note that the confidence interval is tighter than the interval
$$ \PR(\theta) - L_\theta\|\theta-\theta'\|-L_z\varepsilon\|\theta-\theta'\|^\alpha\leq \PR(\theta')\leq \PR(\theta) +L_\theta\|\theta-\theta'\|+L_z\varepsilon\|\theta-\theta'\|^\alpha$$
which holds under the assumption that  $f(\cdot, z)$ is $L_\theta$-Lipschitz continuous for any fixed $z\in \cZ$ where $\cZ$ denotes the domain of the stochastic parameter $z$~\citep{jagadeesan22a}. Here, the latter interval can be deduced by a black-box evaluation of $\PR(\theta)$ while we derived the former using the performative feedback. 

When we have a set $\cS$ of multiple decisions $\theta$ with known $\cD(\theta)$, then for any $\theta'\in \Theta$,
$$\max_{\theta\in\cS}\left\{\DPR(\theta,\theta') - L_z\varepsilon\|\theta-\theta'\|^\alpha\right\}\leq \PR(\theta')\leq \min_{\theta\in\cS}\left\{\DPR(\theta,\theta') +L_z\varepsilon\|\theta-\theta'\|^\alpha\right\}$$
is also valid, and we refer to the bounds as \emph{performative confidence bounds}. The zooming algorithm of \cite{jagadeesan22a} updates the performative confidence bounds whenever a new decision is deployed, based on which suboptimal decisions are sequentially deleted. This approach, however, has two key limitations. First, we need to know $L_z$ and $\varepsilon$ to derive performative confidence bounds. Second, the computational complexity of computing the performative confidence bounds associated with $\cS$ is $O(|\cS|\cdot|\Theta|)$, which is an expensive per-time cost. Later, our experimental results reveal that the algorithm turns out to be not numerically efficient. 

\section{Optimistic Optimization-Based Parameter-Free Framework}\label{sec:full-feedback}

Motivated by the challenges of the existing method, our goal is to design an efficient parameter-free framework for performative risk minimization. For simple presentation, we assume that the loss function and the decision domain. 

\begin{assumption}[bounded domain and objective]\label{assumption:bounded}
$\Theta\subseteq[0,1]^{\Dtheta}$ where $\Dtheta$ is the ambient dimension. Moreover, $f(\theta,z)\in[0,1]$ for all $\theta\in \Theta$ and $z\in\cZ$.
\end{assumption}

Let us explain the basic setup of our optimistic optimization-based framework as follows. We assume that a \emph{hierarchical partitioning}~\citep{x-armed-bandits,munos-soo} of the decision domain is given. Basically, a hierarchical partitioning $\cP$ of $\Theta$ is given by $\{\cP_{h,i}: 0\leq h\leq \hmax, i\in[I_h]\}$ where $\hmax$ is the deepest depth, $I_h$ is the width at depth $h$ with $I_1=1$, $\{\cP_{h,i}:i\in[I_h]\}$ is a partition of $\Theta$, and $\cP_{h,i}$ is partitioned into $\{\cP_{h+1,j}:j\in J\}$ for some $J\subseteq [I_{h+1}]$. Throughout the paper, we refer to $\cP_{h,i}$ as a cell of depth $h$. The hierarchical partitioning naturally corresponds to a tree structure. When a cell $\cP_{h,i}$ is partitioned into $\{\cP_{h+1,j}:j\in J\}$, $\cP_{h,i}$ is the parent cell of its child cells $\cP_{h+1,j}$ for $j\in J$. Moreover, we assume that the partition at each depth level consists of cells of uniform size. 
\begin{assumption}[uniform partition]\label{assumption:partition}
$\sup\{\|\theta-\theta'\|:\theta,\theta'\in\cP_{h,i}\}\leq \sqrt{\Dtheta} 2^{-h}$ for $0\leq h\leq \hmax$ and $1\leq i\leq I_h$. Moreover, $I_h\leq 2^{\Dtheta h}$.
\end{assumption}
There exists a hierarchical partitioning that satisfies~\Cref{assumption:partition} as we may take $2^{\Dtheta}$ subsets of box $[0,1]^{\Dtheta}$ by dividing each coordinate direction equally and repeat the process for each subset.

Our framework adopts the tree-search mechanism as done for many optimistic optimization algorithms such as \SOO, \StoSOO, \SequOOL, and \StroquOOL. These algorithms select an arbitrary decision $\theta_{h,i}$ for each cell $\cP_{h,i}$ as its representative in advance, and evaluating cell $\cP_{h,i}$ means deploying decision $\theta_{h,i}$. If $f(\cdot, z)$ is $L_\theta$-Lipschitz continuous for any $z$, then Assumption~\ref{assumption:partition} implies that for any $\theta_{h,i}\in \cP_{h,i}$,
$$\PR(\theta_{h,i})- \inf_{\theta\in \cP_{h,i}} \PR(\theta)\leq L_\theta \sqrt{D} 2^{-h} + L_z\varepsilon D^{\alpha/2} 2^{-\alpha h}.$$
However, this bound may be too weak for our setting because the term $L_\theta \sqrt{D} 2^{-h}$ can be much larger than the other term $L_z\varepsilon D^{\alpha/2} 2^{-\alpha h}$ when the sensitivity parameter $\varepsilon$ is small. In contrast, based on performative feedback, we use a specific rule for choosing a representative decision $\theta_{h,i}$ given by
$$\theta_{h,i}\in \argmin_{\theta\in \cP_{h,i}} \DPR(\theta_{h-1,j},\theta).$$
Here, $\theta_{h-1,j}$ is the representative of the parent cell $\cP_{h-1,j}$ of depth $h-1$ containing $\cP_{h,i}$. Based on the performative feedback about decision $\theta_{h-1,j}$, we may compute the decoupled performative risk $\DPR(\theta_{h-1,j},\theta)$. Note that the procedure of choosing $\theta_{h,i}$ is much cheaper than computing performative confidence bounds because the former requires evaluating decisions in a local cell $\cP_{h,i}$ whereas the latter considers the entire domain $\Theta$.

Explaining the important components of our framework, we present \Cref{algorithm0} for the full-feedback setting. $\overline{\log}T$ denotes the $T$-th harmonic number, that is, $\overline{\log}T=\sum_{t=1}^T 1/t$. We use notation $[a:b]$ to denote the set $\{a, a+1,\ldots, b\}$ for integers $a,b$ with $a<b$.
\begin{algorithm}[tb]
   \caption{Deterministic Optimistic Optimization with Performative Feedback (DOOP)}\label{algorithm0}
\begin{algorithmic}
   \STATE {\bfseries Input:} test budget $T$, hierarchical partitioning $\cP$,
   $\hmax = \left\lfloor{T}/2^{\Dtheta}\overline{\log}T\right\rfloor$
   \STATE Set $\cL_0\leftarrow \{\cP_{0,1}\}$ and initialize $\cL_1\leftarrow \emptyset$
   \STATE Take a solution $\theta_{0,1}\in \cP_{0,1}$
   \STATE {\bfseries Run} $\mathrm{Open}(\cP_{0,1})$ 
   \FOR{$h=1$ {\bfseries to} $\hmax$}
   \STATE Initialize $\cL_{h+1}\leftarrow \emptyset$
    \STATE Take $\lfloor \hmax/h\rfloor$ cells with the $\lfloor \hmax/h\rfloor$ smallest values in $\left\{\PR(\theta_{h,i}):\cP_{h,i}\in\cL_h\right\}$
    \FOR{each $\cP_{h,i}$ of the $\lfloor \hmax/h\rfloor$ cells}
    \STATE {\bfseries Run} $\mathrm{Open}(\cP_{h,i})$
    \ENDFOR
   \ENDFOR
   \STATE Take $(h,i)\in \argmin_{(h,i)}\left\{\PR(\theta_{h,i}):h\in[0:\hmax+1],\cP_{h,i}\in\cL_h\right\}$
   \STATE {\bfseries Return} $\theta_T\leftarrow\theta_{h,i}$
\end{algorithmic}
\end{algorithm}
\begin{algorithm}[tb]
\renewcommand\thealgorithm{}
\floatname{algorithm}{Subroutine}
   \caption{$\mathrm{Open}(\cP_{h,i})$}
\begin{algorithmic}
\STATE {\bfseries Input:} cell $\cP_{h,i}$
\FOR{each child cell $\cP_{h+1,j}$ of $\cP_{h,i}$}
\STATE Take a solution $\theta_{h+1,j}\in \argmin_{\theta\in \cP_{h+1,j}}\DPR(\theta_{h,i},\theta)$ and deploy it  
\STATE Receive $\cD(\theta_{h+1,j})$ to compute $\DPR(\theta_{h+1,j},\theta)$ for $\theta\in \cP_{h+1,j}$
\STATE Update $\cL_{h+1}\leftarrow \cL_{h+1}\cup \{\cP_{h+1,j}\}$
\ENDFOR
\end{algorithmic}
\end{algorithm}
We adopt \SequOOL by~\cite{bartlett19a} as the backbone of \Cref{algorithm0}. As \SequOOL, our algorithm explores the depth sequentially by testing multiple cells at the same depth level before going down to the next level. As going deeper, fewer cells are tested, thus focusing on a narrower area. This can be viewed as an exploration-exploitation procedure. Moreover, opening a cell $\cP_{h,i}$ of depth $h$ means considering its child cells $\{\cP_{h+1,j}:j\in J\}$ at $h+1$ by deploying their representative decisions $\theta_{h+1,j}$. In \Cref{algorithm0}, $\cL_h$ denotes the set of cells $\cP_{h,i}$ of depth $h$ whose representative decision $\theta_{h,i}$ has been deployed. Then $\cL_0,\ldots, \cL_{\hmax +1}$ naturally form a tree whose vertices correspond to cells.

To analyze the performance of~\Cref{algorithm0}, we use the notion of near-optimality dimension, as mentioned in the introduction. Its definition has been refined, and we adopt the version considered by~\cite{Grill15,bartlett19a}, that is, the near-optimality dimension associated with a given hierarchical partitioning.
\begin{definition}[near-optimality dimension]\label{def:near-optimality}
For any $\nu>0$, $C\geq 1$, and $\rho\in(0,1)$, the \emph{$(\nu, \rho, C)$-near-optimality dimension}, denoted $d(\nu, \rho, C)$, of $f$ with respect to the hierarchical partitioning $\cP$ is defined as
$$d(\nu, \rho, C)=\inf\left\{d\in\mathbb{R}_+: \cN_h(6\nu\rho^h)\leq C\rho^{-dh}\ \forall h\geq 0\right\}$$
where $\cN_h(\epsilon)$ is the number of cells $\cP_{h,i}$ of depth $h$ such that $\inf_{\theta\in \cP_{h,i}}\PR(\theta)\leq \PR(\thetaPO)+\epsilon$.
\end{definition}
In particular, we will use the $((2\sqrt{\Dtheta})^\alpha L_z\varepsilon, 2^{-\alpha}, 1)$-near-optimality dimension. It gets large as the sensitivity parameter $\varepsilon$ increases. Note that by~\Cref{assumption:partition}, the number of cells of depth $h$ is at most $2^{\Dtheta h}$. This gives rise to a global upper bound on $d(\nu, 2^{-\alpha}, 1)$ that holds for any $\nu>0$, that is, $d(\nu, 2^{-\alpha}, 1)\leq \Dtheta/\alpha$. Hence, when $\varepsilon$ is small and $\PR(\cdot)$ has sufficient curvature around the performative-optimal solution $\thetaPO$, the $((2\sqrt{\Dtheta})^\alpha L_z\varepsilon, 2^{-\alpha}, 1)$-near-optimality dimension is supposed to be much smaller than $\Dtheta/\alpha$. When the ambient dimension $D$ is fixed, one may regard the factor $(2\sqrt{\Dtheta})^\alpha$ as a fixed constant and hide it by replacing $\cN_h(6\nu\rho^h)$ with $\cN_h(6(2\sqrt{\Dtheta})^\alpha\nu\rho^h)$ in the definition of $d(\nu, \rho, C)$.

The following lemma is the key to analyzing the performance of \Cref{algorithm0}. Following~\cite{bartlett19a}, we define $\bot_{h}$ as the depth of the deepest cell containing~$\thetaPO$ opened until \Cref{algorithm0} finishes opening cells of depth $h$.
\begin{lemma}\label{lemma:regret-bound0}
Let $d$ denote the $((2\sqrt{\Dtheta})^\alpha L_z\varepsilon, 2^{-\alpha},1)$-near-optimality dimension. Then $\theta_T$ returned by \Cref{algorithm0} satisfies the following bound.
$$\PR(\theta_T)- \PR(\thetaPO)\leq 2(2\sqrt{\Dtheta})^\alpha L_z\varepsilon2^{-\alpha (\bot_{\hmax}+1)}.$$
\end{lemma}
Note that the bound on the optimality gap scales with $L_z\varepsilon$, not $L_\theta + L_z\varepsilon$. Based on this, we prove the following theorem which provides a theoretical guarantee on the performance of \Cref{algorithm0}. As in the analysis of \SequOOL by \cite{bartlett19a}, we use the \emph{Lambert $W$ function}. The function is to describe the solution $h$ to the equation $x=h\cdot e^h$ as $h=W(x)$.
\begin{theorem}\label{thm:full-feedback}
Let $d$ denote the $((2\sqrt{\Dtheta})^\alpha L_z\varepsilon, 2^{-\alpha},1)$-near-optimality dimension. For the full-feedback setting, \Cref{algorithm0} after $T$ decision deployments finds a solution $\theta$ with
$$\PR(\theta) - \PR(\thetaPO) \leq \begin{cases}
2(2\sqrt{\Dtheta})^\alpha L_z\varepsilon2^{-\alpha \hmax},& \text{if $d=0$},\\
2(2\sqrt{\Dtheta})^\alpha L_z\varepsilon e^{-(1/d)W\left(\hmax \alpha d\log 2\right)}, &\text{if $d>0$}
\end{cases}$$
where $\hmax = \left\lfloor T/2^D\overline{\log}T\right\rfloor$. Moreover, if $d>0$ and $\hmax \alpha d \log 2\geq e$, then $\theta$ satisfies
$$\PR(\theta) - \PR(\thetaPO) \leq 2(2\sqrt{\Dtheta})^\alpha L_z\varepsilon\left(\frac{\hmax \alpha d \log 2}{\log(\hmax \alpha d \log 2)}\right)^{-\frac{1}{d}}.$$
\end{theorem}
As $\hmax = \Omega(T/\log T)$ and $\hmax = O(T/\log T)$, we deduce from \Cref{thm:full-feedback} with $\alpha=1$ the optimality gap bounds in \Cref{thm-informal:full-feedback}. We provide the proof of the theorem in \Cref{appendix:full-feedback-proof}. We follow the proof outline of \cite{bartlett19a} for \SequOOL, but we need to adapt the analysis to our specific design of the procedure of opening a cell based on performative feedback.

The last remark is that the optimality gap $\PR(\theta) - \PR(\thetaPO)$ is the \emph{simple regret} whereas \cite{jagadeesan22a} studies the \emph{cumulative regret} incurred by their algorithm. Although \Cref{thm:full-feedback} characterizes an upper bound on the simple regret only, we later report our numerical results on the cumulative regret of \Cref{algorithm0}.

\section{Data-Driven Setting}\label{sec:data-driven}

For the data-driven setting, we receive a few data samples as performative feedback, which provides an estimation of the distribution. Through the data samples, we obtain the empirical distribution $\Dhat(\theta)$ after deploying decision $\theta$. Then we may compute the estimator $\DPRhat(\theta,\theta')$ of the decoupled performative risk $\DPR(\theta,\theta')$ for other decisions $\theta'$. Here, controlling the estimation error $|\DPR(\theta,\theta')-\DPRhat(\theta,\theta')|$ is crucial to achieve a better performance. To reduce the error, we evaluate a cell multiple times to obtain enough data samples from the distribution of the representative decision. Following \StroquOOL by~\cite{bartlett19a}, \Cref{algorithm1} implements this idea, extending \Cref{algorithm0} to the data-driven setting.

As \Cref{algorithm0}, \Cref{algorithm1} takes fewer cells at deeper depth levels, thereby implementing the exploration-exploitation trade-off principle. On top of this, the algorithm keeps track of the number of times each cell has been evaluated. When exploring cells at a certain depth, the algorithm distributes the evaluation budget over cells based on how many times they have been evaluated. Among the cells that have been evaluated many times, we focus on a few that have a low performative risk. For the cells that have not been evaluated many times, we distribute the evaluation budget over more cells, among which we encourage exploration. Furthermore, as in \StroquOOL, \Cref{algorithm1} has the cross-validation phase, in which we focus on cells whose representative decision has a low estimated performative risk. 

Let $\nopen_{h,i}$ denote the number of times cell $\cP_{h,i}$ is opened, and let $\ndeploy_{h,i}$ denote the number of times its representative decision $\theta_{h,i}$ is deployed. Note that if $\cP_{h+1,j}$ is a child cell of $\cP_{h,i}$, then we have $\nopen_{h,i}=\ndeploy_{h+1,j}$. Recall that $[a:b]$ denotes the set $\{a, a+1,\ldots, b\}$ for integers $a,b$ with $a<b$. For a positive integer $a$, let $[a]$ denote the set $\{1,\ldots,a\}$. Moreover, as in~\Cref{algorithm0}, $\cL_0,\ldots, \cL_{\hmax +1}$ represent the tree search structure of~\Cref{algorithm1}.
\begin{algorithm}[tb]
\renewcommand\thealgorithm{2}
   \caption{Stochastic Optimistic Optimization with Performative Feedback (SOOP)}\label{algorithm1}
\begin{algorithmic}
   \STATE {\bfseries Input:} test budget $T$, hierarchical partitioning $\cP$,
   $$\hmax = \left\lfloor\frac{T}{2^{\Dtheta+1}(\log_2T+1)^2}\right\rfloor,\quad \pmax = \lfloor\log_2\hmax\rfloor$$
   \STATE \textit{\textbf{{$\hfill\blacktriangleleft$Initialization Phase $\blacktriangleright$}}}
   \STATE Set $\cL_0\leftarrow \{\cP_{0,1}\}$ and initialize $\cL_1\leftarrow \emptyset$ 
   \STATE Take a solution $\theta_{0,1}\in \cP_{0,1}$, deploy it $\hmax$ times,  and  set $\ndeploy_{0,1}\leftarrow\hmax$
   \STATE \textit{\textbf{{$\hfill\blacktriangleleft$ Exploration Phase $\blacktriangleright$}}}
   \STATE {\bfseries Run} $\mathrm{Open}(\cP_{0,1}, \hmax)$ 

   \FOR{$h=1$ {\bfseries to} $\hmax$}
   \STATE Initialize $\cL_{h+1}\leftarrow \emptyset$
    \FOR{$p=\lfloor \log_2(\hmax/h)\rfloor$ {\bfseries down to} $0$}
    \STATE Take $\lfloor \hmax/h2^p\rfloor$ cells that correspond to the $\lfloor \hmax/h2^p\rfloor$ smallest values in $\left\{\PRhat(\theta_{h,i}):\cP_{h,i}\in\cL_h,\nopen_{h,i}=0,\ndeploy_{h,i}\geq 2^p\right\}$
    \FOR{each $\cP_{h,i}$ of the $\lfloor \hmax/h2^p\rfloor$ cells}
    \STATE {\bfseries Run} $\mathrm{Open}(\cP_{h,i},2^p)$
    \ENDFOR
    \ENDFOR
   \ENDFOR
  \STATE\textit{\textbf{{\hfill $\blacktriangleleft$  Cross-validation Phase $\blacktriangleright$}}}
   \FOR{$p=0$ {\bfseries to} $p_{\max}$} 
   \STATE Take $(h,i)\in \argmin_{(h,i)}\left\{\PRhat(\theta_{h,i}):h\in[0:\hmax+1],\cP_{h,i}\in\cL_h,\ndeploy_{h,i}\geq 2^p\right\}$
   \STATE Set $\theta_T(p)\leftarrow \theta_{h,i}$
   \STATE Deploy $h_{\max}$ times solution $\theta_T(p)$ to form $\PRhat(\theta_T(p))$
   \ENDFOR
   \STATE {\bfseries Return} $\theta_T\leftarrow\theta_T(p)$ with $p\in\argmin_{p\in[0:\pmax]}\PRhat(\theta_{T}(p))$
\end{algorithmic}
\end{algorithm}

\begin{algorithm}[tb]
\renewcommand\thealgorithm{}
\floatname{algorithm}{Subroutine}
   \caption{$\mathrm{Open}(\cP_{h,i},n)$}\label{open}
\begin{algorithmic}
\STATE {\bfseries Input:} cell $\cP_{h,i}$,  number $n$
\STATE Set $\nopen_{h,i}\leftarrow n$
\FOR{each child cell $\cP_{h+1,j}$ of $\cP_{h,i}$}
\STATE Take $\theta_{h+1,j}\in \argmin_{\theta\in \cP_{h+1,j}}\DPRhat(\theta_{h,i},\theta)$, deploy it $n$ times, and set $\ndeploy_{h+1,j}\leftarrow n$
\STATE Form $\DPRhat(\theta_{h+1,j},\theta)$ for $\theta\in \cP_{h+1,j}$
\STATE Initialize $\nopen_{h+1,j}\leftarrow0$
\STATE Update $\cL_{h+1}\leftarrow \cL_{h+1}\cup \{\cP_{h+1,j}\}$
\ENDFOR
\end{algorithmic}
\end{algorithm}

In what follows, we analyze the performance of~\Cref{algorithm1}. Note that we compute $\DPRhat(\theta,\theta')$ for many pairs of $\theta$ and $\theta'$, and at the same time, we need the estimation error $|\DPR(\theta,\theta')-\DPRhat(\theta,\theta')|$ uniformly bounded for all pairs. To achieve this, we introduce the Rademacher complexity associated with the loss function $f$.
\begin{definition}[Rademacher complexity] Given an objective function $f$, the \emph{Rademacher complexity} $\mathfrak{C}^*(f)$ is defined as
\[\mathfrak{C}^*(f) 
= \sup_{\theta \in \Theta} \;\sup_{n \in\mathbb{N}}\; \sqrt{n} \cdot \bbE_{\epsilon, z^{\theta}} \left[\sup_{\theta' \in \Theta} \left|\frac{1}{n} \sum_{j=1}^n \epsilon_j f(\theta', z_j^{\theta})\right|\right],\]
where $\epsilon_j\sim \mathrm{Rademacher}$ and $z_j^{\theta}\sim \cD(\theta)$ for $j\in[n]$, which are all independent of each other.
\end{definition}
Given the Rademacher complexity of the loss function, we may provide a uniform upper bound on the estimation error. Let us define the \emph{clean event} under which the estimation error $|\DPR(\theta,\theta')-\DPRhat(\theta,\theta')|$ is uniformly bounded over all pairs.
\begin{definition}[Clean event]
We define the {clean event}, denoted $\Eclean$, as the event that 
$$\sup_{\theta\in \cP_{h,i}}\left|\DPRhat(\theta_{h,i},\theta) - \DPR(\theta_{h,i},\theta)\right|\leq \frac{2\mathfrak{C}^*(f) + 2\sqrt{\log(T/\delta)}}{\sqrt{\ndeploy_{h,i}m_0}}\quad  \forall \cP_{h,i}\in \cL_h,\ \forall h\in[\hmax+1]$$
and
$$\left|\PRhat(\theta_{T}(p)) - \PR(\theta_{T}(p)) \right|\leq \frac{2\mathfrak{C}^*(f) + 2\sqrt{\log(T/\delta)}}{\sqrt{\hmax m_0}}\quad \forall p\in[0:\pmax].$$
\end{definition}
We may prove that the clean event holds with high probability, parameterized by $\delta$.
\begin{lemma}\label{lemma:clean2}
The clean event holds with probability at least $1-\delta$, i.e., $\bbP\left[\Eclean\right]\geq 1-\delta$.
\end{lemma}

Next, we present the key lemma for our analysis. Following~\cite{bartlett19a}, we define $\bot_{h,p}$ as the depth of the deepest cell containing the performative optimal solution~$\thetaPO$ opened for at least $2^p$ times until \Cref{algorithm1} finishes opening cells of depth $h$.
\begin{lemma}\label{lemma:regret-bound}
Assume that the clean event $\Eclean$ holds for some $\delta\in(0,1)$, and 
let $d$ denote the $((2\sqrt{\Dtheta})^\alpha L_z\varepsilon, 2^{-\alpha},1)$-near-optimality dimension $d((2\sqrt{\Dtheta})^\alpha L_z\varepsilon,2^{-\alpha},1)$. Then for any $p\in[0:\pmax]$, the following bound on the regret holds. 
\begin{align*}
&\PR(\theta_T)- \PR(\thetaPO)\\
&\leq 2(2\sqrt{\Dtheta})^\alpha L_z\varepsilon2^{-\alpha (\bot_{\hmax,p}+1)}+\frac{8\mathfrak{C}^*(f) + 8\sqrt{\log(T/\delta)}}{\sqrt{2^pm_0}} + \frac{4\mathfrak{C}^*(f) + 4\sqrt{\log(T/\delta)}}{\sqrt{\hmax m_0}}.
\end{align*}
\end{lemma}
Therefore, to provide an upper bound on the simple regret, it is sufficient to provide upper bounds on the three terms on the right-hand side. In particular, the second term did not appear in the analysis of \StroquOOL by~\cite{bartlett19a}. Nevertheless, we show that under \Cref{algorithm1}, the three terms are controlled, thereby leading to the desired performance guarantees.

We saw that the simple regret of \Cref{algorithm0} behaves differently depending on whether the near-optimality dimension $d$ satisfies $d=0$ or $d>0$. Similarly, the simple regret of \Cref{algorithm1} varies depending on problem parameters. To illustrate, let us define the \emph{low-noise} and \emph{high-noise} regimes. For simplicity, we use notations $\nu$ and $B$ defined as
$$\nu = (2\sqrt{\Dtheta})^\alpha L_z\varepsilon\qquad \text{and}\qquad B=\frac{2\sqrt{2}\left(\mathfrak{C}^*(f) + \sqrt{\log(T/\delta)}\right)}{\sqrt{m_0}}.$$
Here, $\nu$ captures the term $L_z\varepsilon$, and $B$ is related to the estimation error. Intuitively, if $B$ is high, then the estimation error is large.
We define $\tilde h$ as
\begin{align*}
\tilde h&= \frac{1}{\alpha(d+2)\log 2}W\left(\frac{\hmax \nu^2\alpha (d+2)\log2 }{B^2}\right)\\
&=\frac{1}{\alpha(d+2)\log 2}\left(\log\left(\frac{\hmax \nu^2\alpha (d+2)\log2 }{B^2}\right)-\log\log\left(\frac{\hmax \nu^2\alpha (d+2)\log2 }{B^2}\right)\right)+o(1)\\
&=\Omega\left(\log\left(L_z^2\varepsilon^2\frac{ T}{\log T}\right)\right)
\end{align*}
We refer to the case
$B< L_z\varepsilon \cdot 2^{-\alpha \tilde h}$
as the low-noise regime and the case
$B\geq L_z\varepsilon \cdot 2^{-\alpha \tilde h}$
as the high-noise regime.
\begin{theorem}\label{thm:data-driven}
Let $d$ denote the $((2\sqrt{\Dtheta})^\alpha L_z\varepsilon, 2^{-\alpha},1)$-near-optimality dimension. For the data-driven setting, \Cref{algorithm1} after $T$ decision deployments finds a solution $\theta$ that satisfies the following with probability at least $1-\delta$. Under the low-noise regime with $d=0$,
$$\PR(\theta) - \PR(\thetaPO) \leq (2+2\sqrt{2})(2\sqrt{\Dtheta})^\alpha L_z\varepsilon2^{-\alpha \hmax}+ \frac{4\mathfrak{C}^*(f) + 4\sqrt{\log(T/\delta)}}{\sqrt{\hmax m_0}}$$
where $\hmax =\lfloor T/ 2^{D+1}(\log_2T +1)^2\rfloor$. Under the low-noise regime with $d>0$,
$$\PR(\theta) - \PR(\thetaPO) \leq (2+2\sqrt{2})(2\sqrt{\Dtheta})^\alpha L_z\varepsilon e^{-(1/d)W\left(\hmax \alpha d\log 2\right)}+ \frac{4\mathfrak{C}^*(f) + 4\sqrt{\log(T/\delta)}}{\sqrt{\hmax m_0}}.$$
Under the high-noise regime, 
$$\PR(\theta) - \PR(\thetaPO) \leq 6(2\sqrt{\Dtheta})^\alpha L_z\varepsilon2^{-\alpha \tilde h}+ \frac{4\mathfrak{C}^*(f) + 4\sqrt{\log(T/\delta)}}{\sqrt{\hmax m_0}}.$$
Moreover, if $\hmax\geq \max\{1,\ e/\alpha d\log 2,\ B^2e/\nu^2\alpha(d+2)\log 2\}$, then under the low-noise regime,
$$\PR(\theta) - \PR(\thetaPO) \leq \begin{cases}
(2+3\sqrt{2})(2\sqrt{\Dtheta})^\alpha L_z\varepsilon2^{-\alpha \hmax},& \text{if $d=0$},\\
(2+3\sqrt{2})(2\sqrt{\Dtheta})^\alpha L_z\varepsilon\left(\frac{\hmax \alpha d \log 2}{\log(\hmax \alpha d \log 2)}\right)^{-{1}/{d}}, &\text{if $d>0$}.
\end{cases}$$
Lastly, if $\hmax\geq B^2e/\nu^2\alpha(d+2)\log 2$, then under the high-noise regime, 
\begin{align*}
&\PR(\theta) - \PR(\thetaPO) \\
&\leq 6(2\sqrt{\Dtheta})^\alpha L_z\varepsilon\left(\frac{\hmax \nu^2\alpha (d+2)\log2 /B^2}{\log(\hmax \nu^2\alpha (d+2)\log2 /B^2)}\right)^{-\frac{1}{d+2}}+\frac{4\mathfrak{C}^*(f) + 4\sqrt{\log(T/\delta)}}{\sqrt{\hmax m_0}}.
\end{align*}
\end{theorem}
As $\hmax = \Omega(T/\log T)$, $\hmax = O(T/\log T)$, and $\tilde h = \Omega(\log(L_z^2\varepsilon^2 T/\log T))$, we deduce from \Cref{thm:data-driven} with $\alpha=1$ the simple regret bounds in \Cref{thm-informal:data-driven}. We provide the proof of the theorem in \Cref{appendix:data-driven-proof}.

\section{Experiments}\label{sec:numerical}

\begin{figure}
    \centering
    \subfloat[\centering Ackley Function $A(x_1, x_2)$]{{\includegraphics[width=7cm]{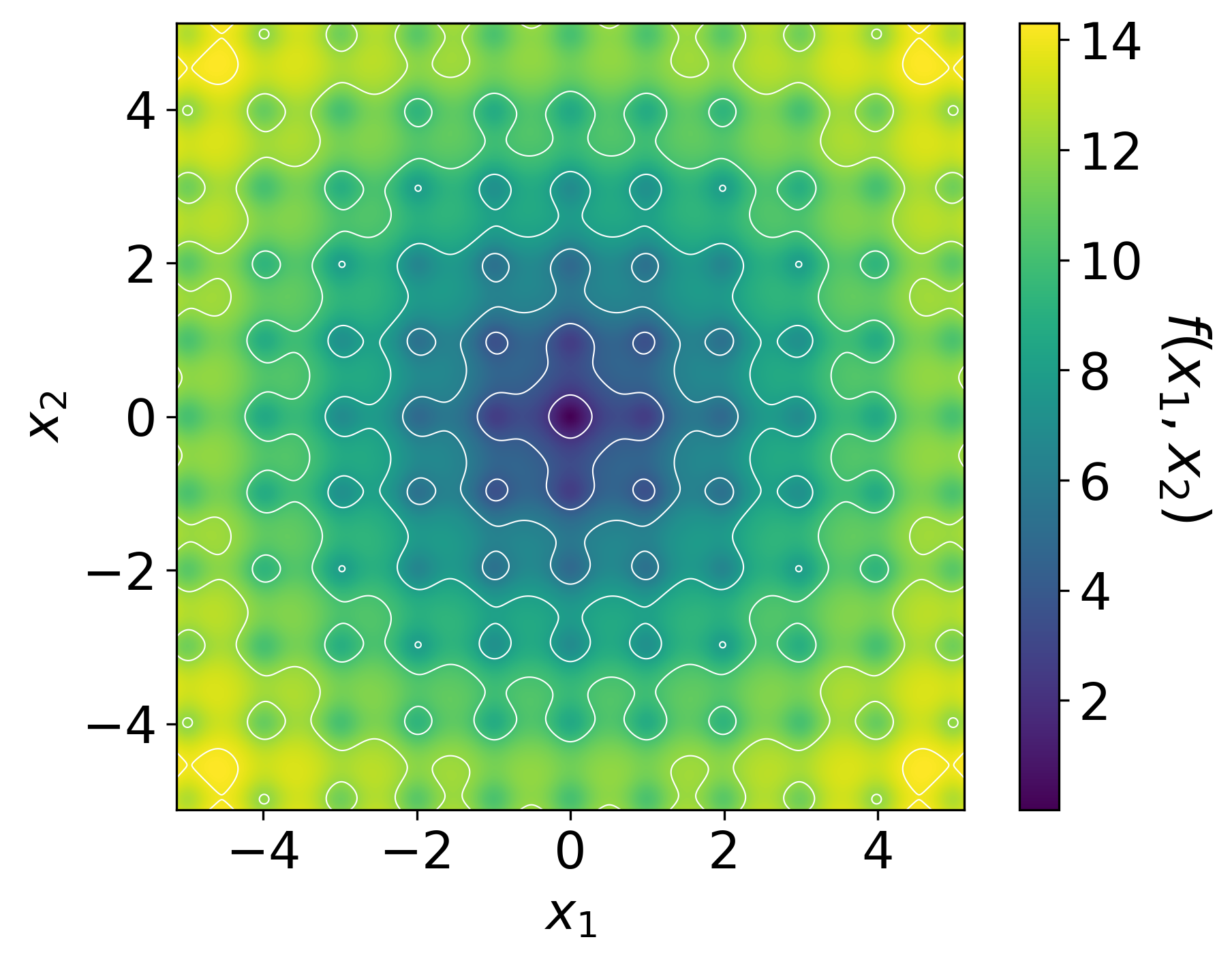} }}
    \qquad
    \subfloat[\centering Rastrigin Function $R(x_1, x_2)$]{{\includegraphics[width=7cm]{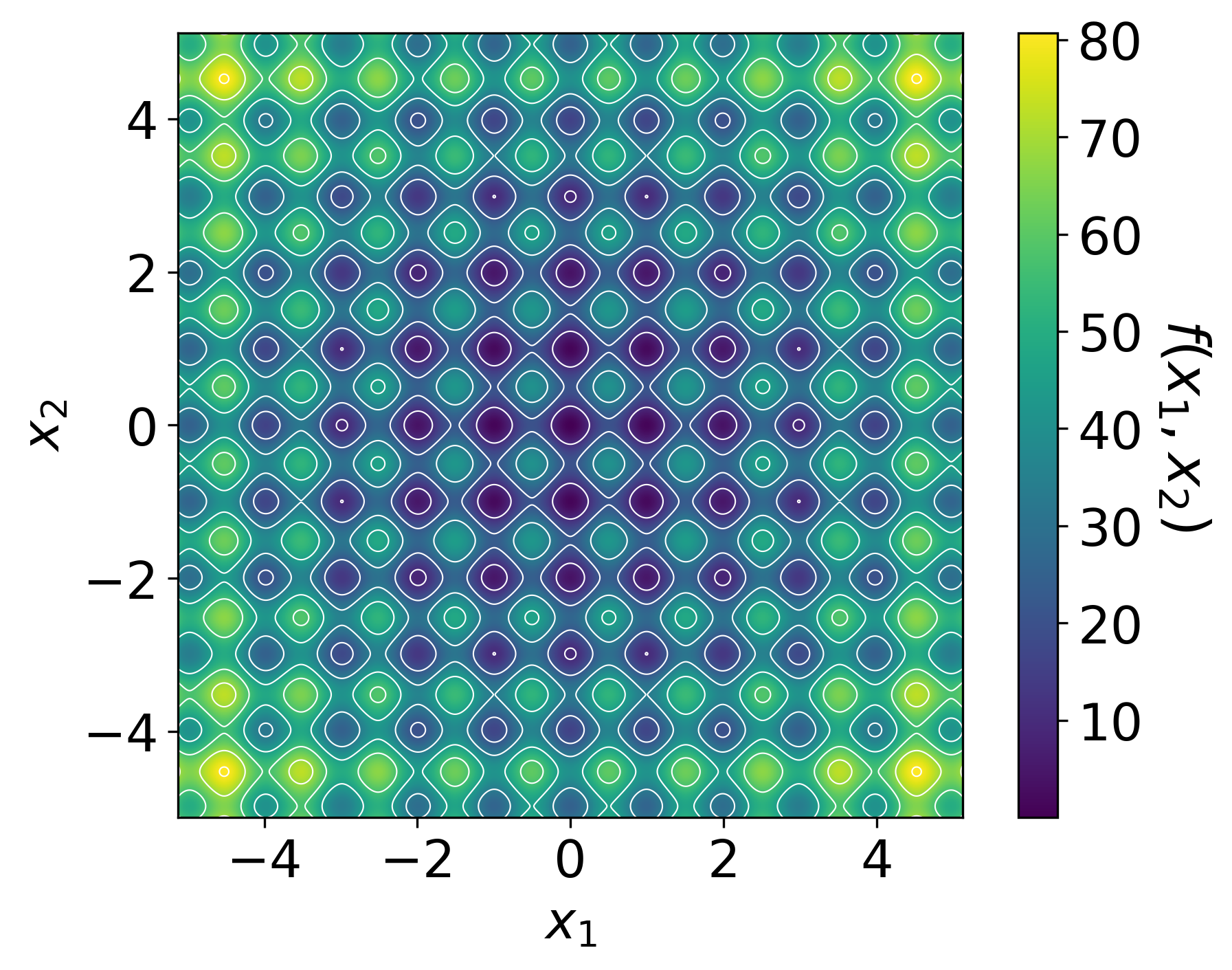} }}
    \caption{Contour plots of the Ackley and Rastrigin functions on $[-5.12, 5.12]^2$}
    \label{fig:syntheticobj}
\end{figure}
In this section, we empirically demonstrate how our algorithms, \DOOP for the full-feedback case and \SOOP for the data-driven setting, numerically perform for performative regret minimization. We compare \DOOP and the existing methods \SOO, \SequOOL, and \SZooming for the full-feedback case, and we test \SOOP against \StoSOO, \StroquOOL, and \SZooming for the data-driven setting. Here, \SZooming indicates the variant of the zooming algorithm by~\cite{jagadeesan22a}. For \SOO, \StoSOO, \SequOOL, and \StroquOOL, we used the package developed by \cite{Li2023PyXAB}.

We tested the algorithms for synthetic objectives on a bounded two-dimensional domain for optimization. In our experiments, we used two multi-modal functions as shown in \Cref{fig:syntheticobj} to express our loss function $f(\theta, z)$ and the distribution map $\cD(\theta)$; the first is the \emph{Ackley function} given by 
$$A(x_1, x_2)=-20\cdot\exp\left[-0.2\sqrt{0.5(x_1^2+x_2^2)}\right]-\exp\left[0.5\left(\cos(2\pi x_1)+\cos(2\pi x_2)\right)\right]$$
and the second is the \emph{Rastrigin function} given by
$$R(x_1, x_2)=20+\left(x_1^2-10\cos(2\pi x_1)\right)+\left(x_2^2-10\cos(2\pi x_2)\right).$$ 
Note that both functions have a global minimum at $A(0, 0)=R(0, 0)=0$, and their domains are both $[-5.12, 5.12]^2$. 
With the Ackley and Rastrigin functions, we define two types of the loss function.
\begin{itemize}
    \item $f(\theta,z)=A(\theta)+z$ with $z\sim \cD(\theta)=\text{Exp}\left({1}/{R(\theta)}\right)$ and $\theta\in[-5.12, 5.12]^2$,
    \item $f(\theta,z)=R(\theta)+z$ with $z\sim \cD(\theta)=\text{Exp}\left({1}/{A(\theta)}\right)$ and $\theta\in[-5.12, 5.12]^2$
\end{itemize}
where $\text{Exp}(1/\lambda)$ denotes the exponential distribution with mean $\lambda$. In both cases, we have
$$\PR(\theta ) = \mathbb{E}_{z\sim \cD(\theta)}[f(\theta,z)] = A(\theta) + R(\theta).$$

\begin{figure}
    \centering
    \subfloat[\centering $f(\theta,z)=A(\theta)+z$ with $z\sim \text{Exp}\left({1}/{R(\theta)}\right)$]{{\includegraphics[width=7cm]{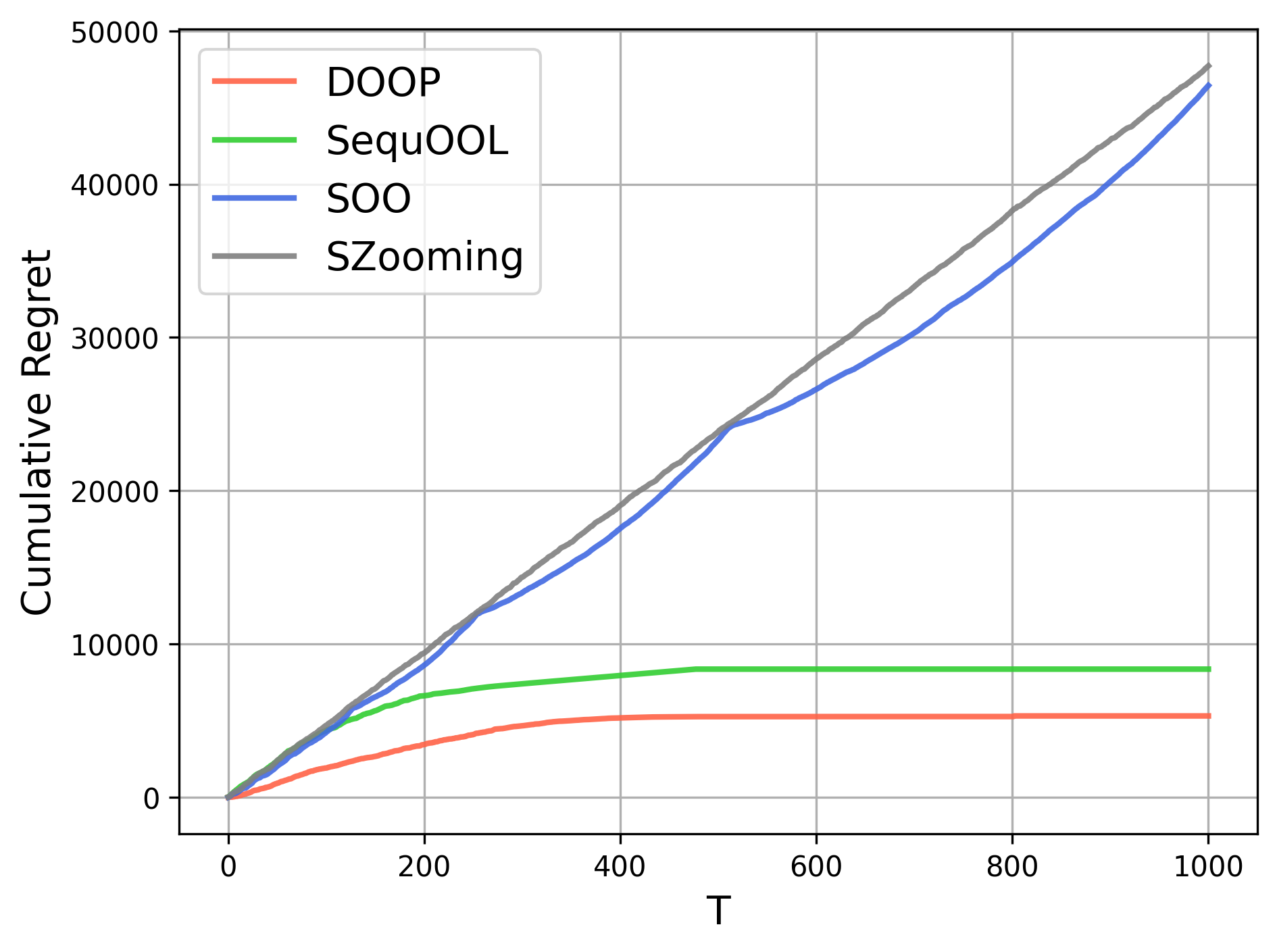} }}
    \qquad
    \subfloat[\centering $f(\theta,z)=R(\theta)+z$ with  $z\sim \text{Exp}\left({1}/{A(\theta)}\right)$]{{\includegraphics[width=7cm]{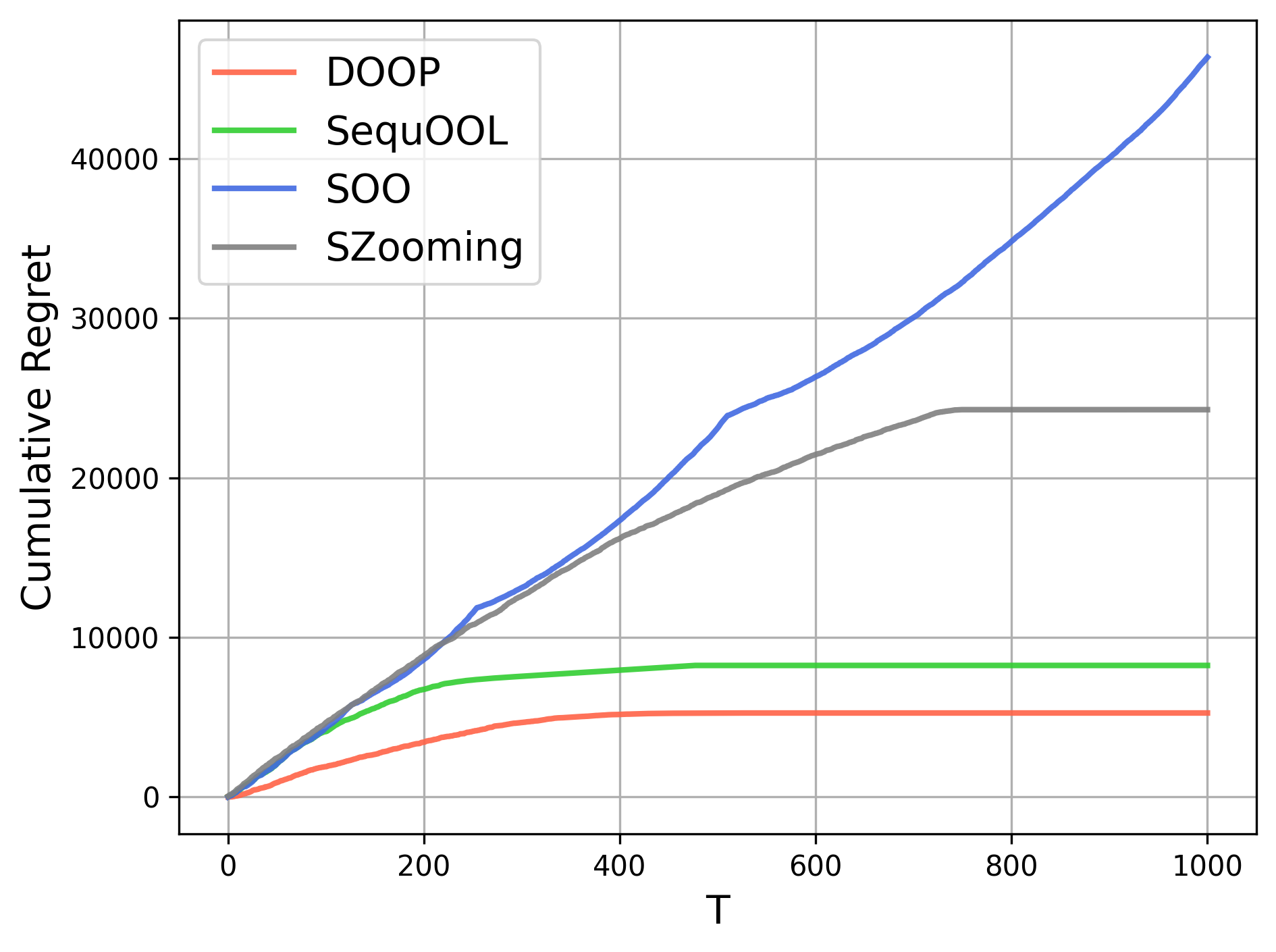} }}
    \caption{Cumulative regret comparison in the full-feedback case}
    \label{fig:deterministic}
\end{figure}
For the full-feedback case, we test \DOOP with \SOO, \SequOOL, and \SZooming, and the performative risk is constructed based on combining the \emph{Ackley function} and the \emph{Rastrigin function}. Recall that the tree search-based algorithms, not including \SZooming, require a hierarchical partitioning, and for them, we used the binary partitioning. For the tree search-based algorithms, we used the same maximum level of depth $h_{\max}$. For \SZooming, the decision domain $\Theta$ is set to be a finite set of 3,025 discrete points on domain $[-5.12, 5.12]^2$. In addition, the sensitivity parameter $\epsilon$ and the Lipschitz constant $L_z$ are chosen according to the shape of the objective functions on the decision domain $\Theta$; both $\epsilon$ and $L_z$ for $z\sim\text{Exp}(1/f(\theta))$ is calculated to be $\sup\left(|f(\theta_1)-f(\theta_2)|/\lVert\theta_1-\theta_2\rVert\right)$ for any $\theta_1, \theta_2\in\Theta$.

\begin{figure}
    \centering
    \subfloat[\centering $f(\theta,z)=A(\theta)+z$ with $z\sim \text{Exp}\left({1}/{R(\theta)}\right)$]{{\includegraphics[width=7cm]{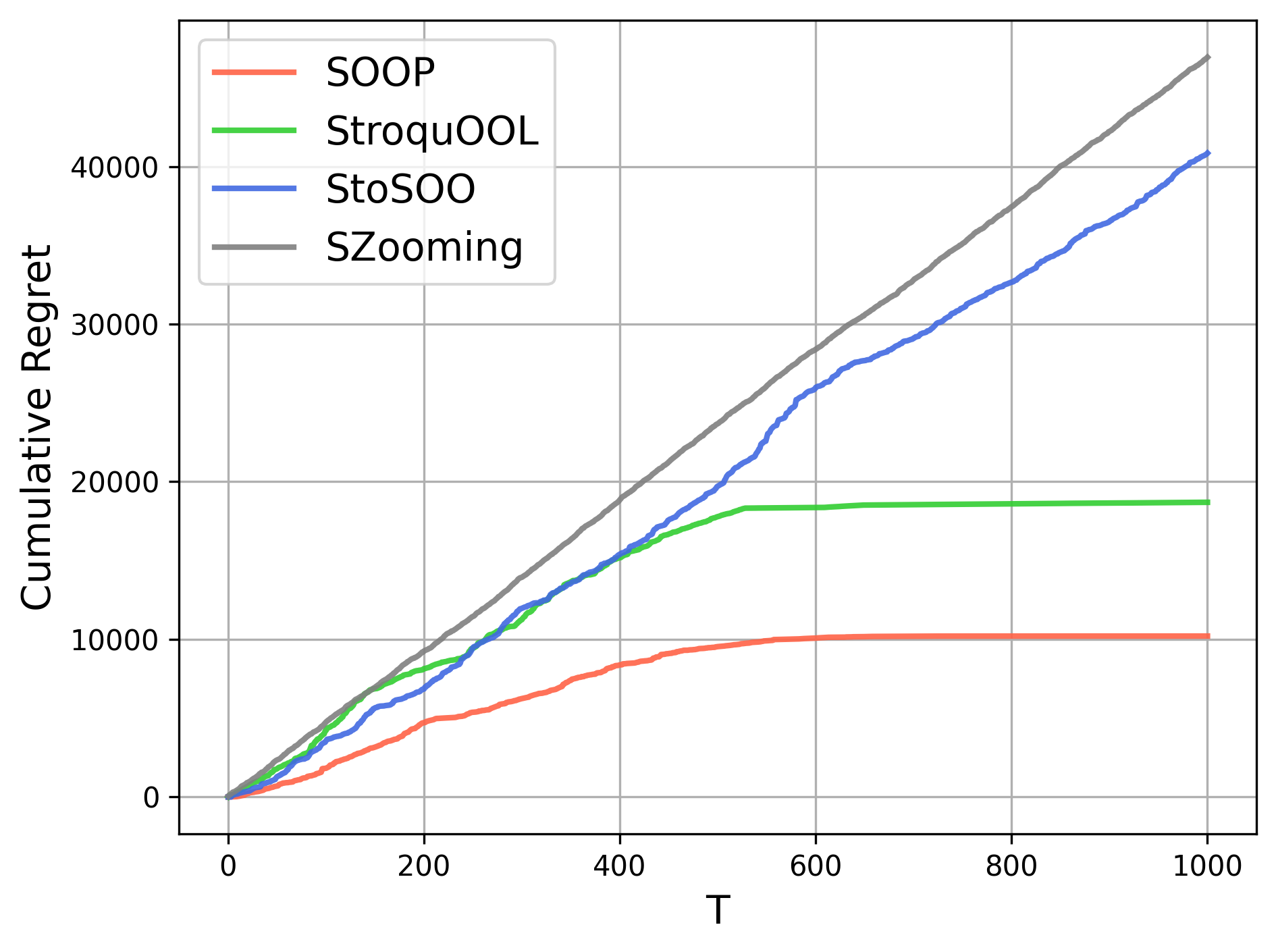} }}
    \qquad
    \subfloat[\centering $f(\theta,z)=R(\theta)+z$ with  $z\sim \text{Exp}\left({1}/{A(\theta)}\right)$]{{\includegraphics[width=7cm]{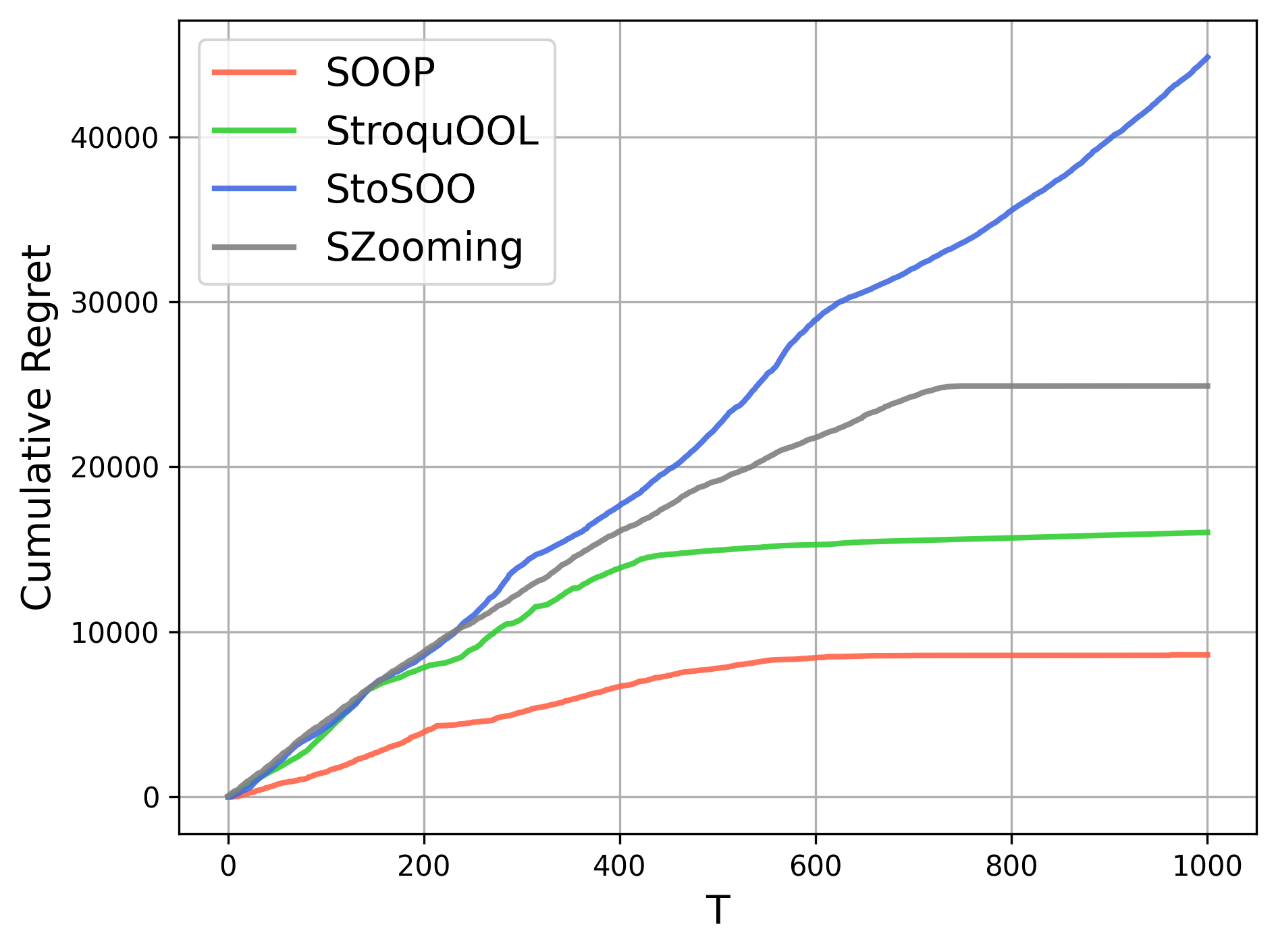} }}
    \caption{Cumulative regret comparison in the data-driven setting}
    \label{fig:stochastic}
\end{figure}
For the data-driven setting, we test \SOOP with \StoSOO, \StroquOOL, and \SZooming. The performative feedback consists of $m_0=10$ samples drawn randomly from $\mathcal{D}(\theta)$. While both $A(\theta)$ and $R(\theta)$ are multi-modal on the given domain, each has different sets of local optima and range. In particular, $R(\theta)$ yields values in a broader range, thus the associated distribution $\text{Exp}\left({1}/{R(\theta)}\right)$ has a larger variance since the variance of $\text{Exp}(1/\lambda)$ is ${\lambda^2}$. The other components of the experimental setup are the same as those for the full-feedback case.

\Cref{fig:deterministic,fig:stochastic} summarize our experimental results. 
As shown in \Cref{fig:deterministic,fig:stochastic},  \DOOP and \SOOP outperform the other methods in terms of cumulative regret. We have noticed that \SZooming incurs a very high cumulative regret in the first phase, and this is because there exist high estimation errors in the first phase of \SZooming and it turns out that the majority of exploration of \SZooming occurs during the first phase. Moreover, \SZooming is not computationally efficient, as it takes a huge amount of time to find an optimal decision. For the full-feedback case, \SZooming takes 73691 seconds for $f(\theta,z)=A(\theta)+z$ with $z\sim \cD(\theta)=\text{Exp}\left({1}/{R(\theta)}\right)$ and 3737.0 seconds for $f(\theta,z)=R(\theta)+z$ with $z\sim \cD(\theta)=\text{Exp}\left({1}/{A(\theta)}\right)$. In contrast, \DOOP takes 4.2390 seconds and 2.3609 seconds, respectively. For the data-driven case, \SZooming takes 73112 seconds for $f(\theta,z)=A(\theta)+z$ with $z\sim \cD(\theta)=\text{Exp}\left({1}/{R(\theta)}\right)$ and 3999.4 seconds for $f(\theta,z)=R(\theta)+z$ with $z\sim \cD(\theta)=\text{Exp}\left({1}/{A(\theta)}\right)$. In contrast, \SOOP takes 2.5896 seconds and 1.4877 seconds, respectively.

\acks{This research is supported, in part, by KAIST Starting Fund (KAIST-G04220016), FOUR Brain Korea 21 Program (NRF-5199990113928), and National Research Foundation of Korea (NRF-2022M3J6A1063021).}

\newpage

\appendix
\section{Regret Analysis of \DOOP}\label{appendix:full-feedback-proof}

\subsection{Approximation Bounds under the Hierarchical Partitioning Scheme}

In this section, we prove two lemmas that are related to the quality of the representative decision of each cell.

\begin{lemma}\label{lemma:PR-approximation}
For $\theta\in\cP_{h,i}$ and $\cP_{h,i}\in\cL_h$, we have
$$\left|\DPR(\theta_{h,i},\theta)-\PR(\theta)\right|\leq (\sqrt{\Dtheta})^\alpha L_z\varepsilon2^{-\alpha h}.$$
\end{lemma}
\begin{proof}
Note that
$$\DPR(\theta_{h,i},\theta) \geq \PR(\theta)-L_z\varepsilon\|\theta_{h,i}-\theta\|^\alpha\geq \PR(\theta)-(\sqrt{\Dtheta})^\alpha L_z\varepsilon2^{-\alpha h}$$
where the first inequality is from Lemma~\ref{lemma:lipschitz} and the second inequality follows from~\Cref{assumption:partition} with $\theta_{h,i},\theta\in\cP_{h,i}$. Similarly, we can argue that
$ \DPR(\theta_{h,i}, \theta)\leq \PR(\theta)+L_z\varepsilon\|\theta_{h,i}-\theta\|^\alpha\leq \PR(\theta)+(\sqrt{\Dtheta})^\alpha L_z\varepsilon2^{-\alpha h},$
as required.
\end{proof}

\begin{lemma}\label{lemma:cell-approximation0}
Let $\cP_{h,i}\in \cL_{h}$. Then 
$$\PR(\theta_{h,i})\leq \inf_{\theta\in \cP_{h,i}}\PR(\theta)+2(2\sqrt{\Dtheta})^\alpha L_z\varepsilon2^{-\alpha h}.$$
\end{lemma}
\begin{proof}
Let $\theta^\star_{h,i}\in \argmin_{\theta\in \cP_{h,i}}\PR(\theta)$, and let $\cP_{h-1,j}$ be the parent cell of $\cP_{h,i}$. Note that
$$\PR(\theta_{h,i}) \leq \DPR(\theta_{h-1,j},\theta_{h,i}) + (2\sqrt{\Dtheta})^\alpha L_z\varepsilon2^{-\alpha h}\leq \DPR(\theta_{h-1,j},\theta^\star_{h,i})+ (2\sqrt{\Dtheta})^\alpha L_z\varepsilon2^{-\alpha h}$$
where the first inequality follows from Lemma~\ref{lemma:PR-approximation} and the second inequality holds due to our choice of $\theta_{h,i}$ minimizing $\DPR(\theta_{h-1,j},\theta)$ over $\theta\in\cP_{h,i}$.
Lastly, by 
Lemma~\ref{lemma:PR-approximation}, we have
$$\DPR(\theta_{h-1,j},\theta^\star_{h,i})\leq \PR(\theta_{h,i}^\star) + (2\sqrt{\Dtheta})^\alpha L_z\varepsilon2^{-\alpha h}.$$
Consequently, it follows that
$\PR(\theta_{h,i})\leq \PR(\theta^\star_{h,i})+ 2(2\sqrt{\Dtheta})^\alpha L_z\varepsilon2^{-\alpha h},$
as required.
\end{proof}

\subsection{Proof of \Cref{thm:full-feedback}}

Recall that $\bot_{h}$ is defined on the depth of the deepest cell containing~$\thetaPO$ opened until \Cref{algorithm0} finishes opening cells of depth $h$.
\begin{lemma}\label{lemma:depth0}
Let $d$ denote the $((2\sqrt{\Dtheta})^\alpha L_z\varepsilon, 2^{-\alpha},1)$-near-optimality dimension. For any $h\in[\hmax]$, if $\hmax/h\geq 2^{\alpha hd}$, then $\bot_{h}=h$ with $\bot_{0}=0$.
\end{lemma}
\begin{proof}
Let $h\in[\hmax]$, and assume that $h$ satisfies the condition of the lemma. Then we will argue by induction that $\bot_{h'}=h'$ for all $h'\in [h]$, thereby proving that $\bot_{h}=h$. 

Note that $\cP_{0,1}=\Theta$ contains $\thetaPO$ and $\cP_{0,1}$ is opened, so $\bot_{0}=0$. Next, we assume that $\bot_{h'-1}=h'-1$ for some $h'\in[h]$. Then it is sufficient to show that $\bot_{h'}=h'$. Let $i^\star_{h'-1}$ denote the index such that $\cP_{h'-1,i^\star_{h'-1}}$ is the cell containing $\thetaPO$ at depth $h'-1$. \Cref{algorithm0} opens $\lfloor \hmax/h'\rfloor$ cells from depth $h'$ cells. Suppose for a contradiction that cell $\cP_{h',i^\star_{h'}}$ is not one of them. This implies that for each solution $\theta_{h',i}$ of the $\lfloor \hmax/h'\rfloor$ cells from depth $h'$, we have $\PR(\theta_{h',i})\leq \PR(\theta_{h',i^\star_{h'}}).$ Consequently, it follows that
$$\PR(\theta_{h',i})\leq \PR(\theta_{h',i^\star_{h'}})\leq \PR(\thetaPO)+2(2\sqrt{\Dtheta})^\alpha L_z\varepsilon2^{-\alpha h'}$$
where the second inequality follows from~Lemma~\ref{lemma:cell-approximation0} as $\thetaPO$ is contained in cell $\cP_{h',i^\star_{h'}}$. 
This implies that
$$\cN_h(6(2\sqrt{\Dtheta})^\alpha L_z\varepsilon2^{-\alpha h'})\geq \left\lfloor\frac{\hmax}{h'}\right\rfloor+1\geq \left\lfloor\frac{\hmax}{h}\right\rfloor+1\geq 2^{\alpha hd}+1\geq 2^{\alpha h'd}+1$$
where $\lfloor{\hmax}/{h'}\rfloor$ comes from cells $\cP_{h',i}$ and 1 is due to cell $\cP_{h',i^\star_{h'}}$ in the first inequality, the second and the fourth inequalities hold because $h'\leq h$, and the third inequality comes from the condition of the lemma. This in turn implies that $\cN_h(6(2\sqrt{\Dtheta})^\alpha L_z\varepsilon2^{-\alpha h'})>2^{\alpha h'd}$, a contradiction. Therefore, it follows that $\bot_{h'}=h'$. Then the induction argument shows that $\bot_{h}=h$, as required.
\end{proof}

Based on Lemma~\ref{lemma:depth0}, we prove Lemma~\ref{lemma:regret-bound0} that shows
$$\PR(\theta_T)- \PR(\thetaPO)\leq 2(2\sqrt{\Dtheta})^\alpha L_z\varepsilon2^{-\alpha (\bot_{\hmax}+1)}.$$
\begin{proof}[{\bfseries Proof of Lemma~\ref{lemma:regret-bound0}}]
Let $\cP_{\bot_{\hmax}+1,i^\star}$ be the cell at depth $\bot_{\hmax}+1$ containing $\thetaPO$.
Note that 
$$\PR(\theta_T)\leq \PR(\theta_{\bot_{\hmax}+1,i^\star}) \leq \PR(\thetaPO)+2(2\sqrt{\Dtheta})^\alpha L_z\varepsilon2^{-\alpha (\bot_{\hmax}+1)}$$
where the first inequality holds due to the choice of $\theta_T$ and the second inequality follows from Lemma~\ref{lemma:cell-approximation0}.
\end{proof}

For simplicity, we introduce notations $\rho$ and $\nu$ defined as 
$$\rho =  2^{-\alpha}\quad\text{and}\quad \nu = (2\sqrt{D})^\alpha L_z\varepsilon.$$
Moreover, we define $\bar h$ as the number satisfying
$$\frac{\hmax}{\bar h} =\rho^{-d\bar h}.$$
Note that if $d=0$, then $\bar h = \hmax$. If $d>0$, then 
$$\bar h = \frac{1}{d\log(1/\rho)}W\left(\hmax d\log(1/\rho)\right)$$
where $W(\cdot)$ denotes the Lambert $W$ function. 
\begin{lemma}[\citet{bartlett19a}]\label{lemma:bartlett0}
Let $d$ denote the $(\nu, \rho,1)$-near-optimality dimension. Then 
$\bot_{\hmax}+1\geq  \bar h.$
\end{lemma}
Combining Lemmas~\ref{lemma:regret-bound0} and~\ref{lemma:bartlett0}, we are ready to provide the desired regret bounds on \Cref{algorithm0}.
\begin{proof}[{\bfseries Proof of \Cref{thm:full-feedback}}]
As $\bar h =\hmax$ when $d=0$ and $\bar h = W(\hmax \alpha d\log 2)/\alpha d\log 2$, it follows directly from Lemmas~\ref{lemma:regret-bound0} and~\ref{lemma:bartlett0} that
$$\PR(\theta) - \PR(\thetaPO) \leq \begin{cases}
2(2\sqrt{\Dtheta})^\alpha L_z\varepsilon2^{-\alpha \hmax},& \text{if $d=0$},\\
2(2\sqrt{\Dtheta})^\alpha L_z\varepsilon e^{-(1/d)W\left(\hmax \alpha d\log 2\right)}, &\text{if $d>0$.}
\end{cases}$$
Lastly, \cite{Hoorfar2008} showed that if $x\geq e$, then $W(x)\geq \log(x/\log (x))$. Hence, if $d>0$ and $\hmax \alpha d \log 2\geq e$, then $\theta$ satisfies
$$\PR(\theta) - \PR(\thetaPO) \leq 2(2\sqrt{\Dtheta})^\alpha L_z\varepsilon\left(\frac{\hmax \alpha d \log 2}{\log(\hmax \alpha d \log 2)}\right)^{-\frac{1}{d}},$$
as required.
\end{proof}

\section{Regret Analysis of \SOOP}\label{appendix:data-driven-proof}

\subsection{Total Number of Solution Deployments}

Recall that
$$\hmax = \left\lfloor\frac{T}{2^{\Dtheta+1}(\log_2T+1)^2}\right\rfloor\quad\text{and}\quad \pmax = \lfloor\log_2\hmax\rfloor.$$

\begin{lemma}\label{lemma:number}
The total number of solution deployments before the cross-validation phase is at most $3T/4$, and in the cross-validation phase, the total number of solution deployments is at most $T/4$.
\end{lemma}
\begin{proof}
Note that as $T\geq 2$, we have $\log_2T+1\geq 2$, in which case
$\hmax\leq {T}/{2^{\Dtheta+3}}.$ Hence, we deploy solution $\theta_{0,1}$ at most $T/2^{\Dtheta+3}\leq T/8$ times. Moreover, we open $\cP_{0,1}$ at most $T/2^{\Dtheta+3}$ times, and since $\cP_{0,1}$ has $2^{\Dtheta}$ child cells, it corresponds to at most ${T}/8$ solution deployments. Next, during the exploration phase, we make  $\sum_{h=1}^{\hmax}\sum_{p=0}^{\pmax}\lfloor{\hmax}/{h2^p}\rfloor 2^p$ openings. Here,
$$\sum_{h=1}^{\hmax}\sum_{p=0}^{\pmax}\left\lfloor \frac{\hmax}{h2^p}\right\rfloor\cdot 2^p\leq \sum_{h=1}^{\hmax}\sum_{p=0}^{\pmax}\frac{\hmax}{h}= (\pmax+1)\hmax\sum_{h=1}^{\hmax}\frac{1}{h}\leq \hmax(\pmax+1)^2\leq \frac{T}{2^{\Dtheta+1}}$$
where the last inequality holds due to $\pmax \leq \log_2T$. Since each opening requires $2^{\Dtheta}$ solution deployments, it incurs $T/2$ solution deployments. In total, before the cross-validation phase, we make $T/8+T/8 + T/2 = 3T/4$ solution deployments. 

In the cross-validation phase, the number of solution deployments is given by
$$\hmax (\pmax + 1)\leq \frac{T}{2^{\Dtheta+1}(\log_2T+1)}\leq \frac{T}{2^{\Dtheta+2}}\leq\frac{T}{4},$$
as required. 
\end{proof}

\subsection{Rademacher Complexity-Based Concentration Bounds for Estimating the Performative Risk}

In this section, we prove Lemma~\ref{lemma:clean2} which shows that the clean event holds with probability at least $1-\delta$, i.e., $\bbP[\Eclean]\geq 1-\delta$.
\begin{proof}[{\bfseries Proof of Lemma~\ref{lemma:clean2}}]
By Lemma~\ref{lemma:number}, \Cref{algorithm1} deploys at most $T$ solutions. Let $\ndistinct$ denote the number of distinct solutions deployed by \Cref{algorithm1}. As new solutions are deployed during the exploration phase, we have $\ndistinct\leq 3T/4$ by Lemma~\ref{lemma:number}. Among the $\ndistinct$ solutions, we use notation $(h(s),i(s))$ to indicate the cell $\cP_{h(s),i(s)}$ containing the $s$th deployed solution for $1\leq s\leq \ndistinct$. As $\hmax$ is fixed, $\ndistinct$, $h(s)$, and $\ndeploy_{h(s),i(s)}$ are all deterministic functions of $s$. In particular, we use notation $\ndeploy_s:=\ndeploy_{h(s),i(s)}$ to emphasize that $\ndeploy_{h(s),i(s)}$ is deterministic in $s$. Then we maintain a \emph{virtual tape of samples} for each solution $\theta$. Basically, for each solution $\theta$, we maintain $\{z_j^{\theta}:j=1,\ldots, T m_0\}$, %
and if $\theta$ becomes the $s$th solution deployed, then we use $\ndeploy_sm_0$ samples in $\{z_j^{\theta}:j=(\ndeploy_{1}+\cdots +\ndeploy_{s-1})m_0+1,\ldots, (\ndeploy_{1}+\cdots +\ndeploy_{s})m_0\}$ to estimate $\Dhat(\theta)$.
For $1\leq s\leq \ndistinct$, let us define $\Eclean^s$ as the event that 
$$\sup_{\theta\in \cP_{s}}\left|\DPRhat(\theta_{s},\theta) - \DPR(\theta_{s},\theta)\right|\leq \frac{2\mathfrak{C}^*(f) + 2\sqrt{\log(T/\delta)}}{\sqrt{\ndeploy_sm_0}}$$
where $\theta_s$ denotes the $s$th solution deployed $\theta_{h(s),i(s)}$, $\cP_s$ denotes $\cP_{h(s),i(s)}$ containing the $s$th solution deployed, and $\ndeploy_s$ denotes the number of times solution $\theta_{h(s),i(s)}$ is deployed. Moreover, for $p\in[0:\pmax]$, let us define $\Eclean^{T,p}$ as the event that
$$\left|\PRhat(\theta_{T}(p)) - \PR(\theta_{T}(p)) \right|\leq \frac{2\mathfrak{C}^*(f) + 2\sqrt{\log(T/\delta)}}{\sqrt{\hmax m_0}}.$$
Then we know that 
\begin{align*}
\bbP[\Eclean]&=\bbP[\Eclean^1\cap\cdots\cap \Eclean^{\ndistinct}\cap \Eclean^{T,0}\cap\cdots \cap \Eclean^{T,\pmax}]\\
&\geq 1 - \sum_{s=1}^{\ndistinct} \bbP[\neg \Eclean^s]-\sum_{p=0}^{\pmax} \bbP[\neg \Eclean^{T,p}]
\end{align*}
where the inequality is the union bound. For simplicity, let $J$ denote $J=\{(\ndeploy_{1}+\cdots +\ndeploy_{s-1})m_0+1,\ldots, (\ndeploy_{1}+\cdots +\ndeploy_{s})m_0\}$. Note that
\begin{align*}
\bbP[\neg \Eclean^s]
&=\bbP\left[\sup_{\theta\in \cP_{s}}\left|\DPRhat(\theta_{s},\theta) - \DPR(\theta_{s},\theta)\right|> \frac{2\mathfrak{C}^*(f) + 2\sqrt{\log(T/\delta)}}{\sqrt{\ndeploy_sm_0}}\right]\\
&=\bbP\left[\sup_{\theta\in \cP_{s}}\left|\frac{1}{\ndeploy_sm_0}\sum_{j\in J}f(\theta,z_{j}^{\theta_s}) - \DPR(\theta_{s},\theta)\right|> \frac{2\mathfrak{C}^*(f) + 2\sqrt{\log(T/\delta)}}{\sqrt{\ndeploy_sm_0}}\right]\\
&\leq \bbP\left[\sup_{\theta\in \Theta}\left|\frac{1}{\ndeploy_sm_0}\sum_{j\in J}f(\theta,z_{j}^{\theta_s}) - \DPR(\theta_{s},\theta)\right|> \frac{2\mathfrak{C}^*(f) + 2\sqrt{\log(T/\delta)}}{\sqrt{\ndeploy_sm_0}}\right]
\end{align*}
where the inequality holds because $\cP_s\subseteq \Theta$. Here, the right-most side of this inequality is equal to
$$\bbE_{\bar \theta\sim \theta_s}\left[\bbP\left[\sup_{\theta\in \Theta}\left|\frac{1}{\ndeploy_sm_0}\sum_{j\in J}f(\theta,z_{j}^{\bar\theta}) - \DPR(\bar\theta,\theta)\right|> \frac{2\mathfrak{C}^*(f) + 2\sqrt{\log(T/\delta)}}{\sqrt{\ndeploy_sm_0}}\mid \theta_s=\bar \theta\right] \right].$$
Therefore, to provide an upper bound on $\bbP[\neg \Eclean^s]$, it suffices to provide an upper bound on
\begin{equation}\label{eq:clean1}\bbP\left[\sup_{\theta\in \Theta}\left|\frac{1}{\ndeploy_sm_0}\sum_{j\in J}f(\theta,z_{j}^{\bar\theta}) - \DPR(\bar\theta,\theta)\right|> \frac{2\mathfrak{C}^*(f) + 2\sqrt{\log(T/\delta)}}{\sqrt{\ndeploy_sm_0}}\mid \theta_s=\bar \theta\right]\end{equation}
for every $\bar \theta \in\Theta$. Note that data samples in $\{z_j^{\bar \theta}:j=(\ndeploy_{1}+\cdots +\ndeploy_{s-1})m_0+1,\ldots, (\ndeploy_{1}+\cdots +\ndeploy_{s})m_0\}$ are independent of the event that $\theta_s = \bar \theta$ because the samples are obtained after the $s$th solution for deployment is chosen. Therefore, the probability term~\eqref{eq:clean1} is equal to 
\begin{equation}\label{eq:clean2}
\bbP\left[\sup_{\theta\in \Theta}\left|\frac{1}{\ndeploy_sm_0}\sum_{j\in J}f(\theta,z_{j}^{\bar\theta}) - \DPR(\bar\theta,\theta)\right|> \frac{2\mathfrak{C}^*(f) + 2\sqrt{\log(T/\delta)}}{\sqrt{\ndeploy_sm_0}}\right].\end{equation}
What remains is to bound this probability term for every $\bar \theta\in \Theta$. By the bounded differences inequality and \Cref{assumption:bounded}, with probability at least $1-(\delta/T)$, we have
\begin{align}\label{eq:clean3}
\begin{aligned}
&\sup_{\theta\in \Theta}\left|\frac{1}{\ndeploy_sm_0}\sum_{j\in J}f(\theta,z_{j}^{\bar\theta}) - \DPR(\bar\theta,\theta)\right|\\
&\leq \bbE\left[\sup_{\theta\in \Theta}\left|\frac{1}{\ndeploy_sm_0}\sum_{j\in J}f(\theta,z_{j}^{\bar\theta}) - \DPR(\bar\theta,\theta)\right|\right]+\sqrt{\frac{2 \log(T/\delta)}{\ndeploy_s m_0}}.
\end{aligned}
\end{align}
Let $\epsilon_j$ denote i.i.d. Rademacher random variables. Then by a symmetrization argument, the right-hand side of~\eqref{eq:clean3} is at most 
\begin{align}\label{eq:clean4}
\begin{aligned}
&\bbE\left[\sup_{\theta\in \Theta}\left|\frac{1}{\ndeploy_sm_0}\sum_{j\in J}f(\theta,z_{j}^{\bar\theta}) - \DPR(\bar\theta,\theta)\right|\right]+\sqrt{\frac{2 \log(T/\delta)}{\ndeploy_s m_0}}\\
&\leq 2\cdot \bbE\left[\sup_{\theta\in \Theta}\left|\frac{1}{\ndeploy_sm_0}\sum_{j\in J}f(\theta,z_{j}^{\bar\theta})\cdot \epsilon_j\right|\right]+\sqrt{\frac{2 \log(T/\delta)}{\ndeploy_s m_0}}\\
&\leq \frac{2}{\sqrt{\ndeploy_sm_0}}\cdot \sup_{n\geq1}\sqrt{n}\cdot \bbE\left[\sup_{\theta\in \Theta}\left|\frac{1}{n}\sum_{j=1}^{n}f(\theta,z_{j}^{\bar\theta})\cdot \epsilon_j\right|\right]+\sqrt{\frac{2 \log(T/\delta)}{\ndeploy_s m_0}}\\
&\leq \frac{2 \mathfrak{C}^*(f) + 2\sqrt{\log(T/\delta)}}{\sqrt{\ndeploy_s m_0 }}
\end{aligned}
\end{align}
where $\{z_j^{\theta}\}_{j\in\mathbb{N}}$ denotes an infinite sequence of samples from $\cD(\theta)$ for $\theta\in \Theta$.
By~\eqref{eq:clean3} and~\eqref{eq:clean4}, the probability term~\eqref{eq:clean2} as well as~\eqref{eq:clean1} is at most $\delta/T$. Therefore, it follows that $\bbP[\neg \Eclean^s]\leq \delta/T$.

Next, we consider $\bbP[\neg \Eclean^{T,p}]$. Note that the total number of solution deployments during the exploration phase, denoted $\ntotal$ is deterministic. Note that $\theta_T(p)$ is deployed for $\hmax$ times and obtain samples $\{z_j^{\theta_T(p)}: j\in J' \}$ where $J'=\{(\ntotal+p\hmax)m_0+1,\ldots,(\ntotal+(p+1)\hmax)m_0\}$. Then we have
\begin{align*}
&\bbP[\neg \Eclean^{T,p}]\\
&=\bbP\left[\left|\PRhat(\theta_{T}(p)) - \PR(\theta_{T}(p)) \right|>\frac{2\mathfrak{C}^*(f) + 2\sqrt{\log(T/\delta)}}{\sqrt{\hmax m_0}}\right]\\
&=\bbP\left[\left|\frac{1}{\hmax m_0}\sum_{j\in J'}f(\theta_T(p),z_{j}^{\theta_T(p)}) - \DPR(\theta_{T}(p),\theta_T(p))\right|> \frac{2\mathfrak{C}^*(f) + 2\sqrt{\log(T/\delta)}}{\sqrt{\hmax m_0}}\right]\\
&\leq \bbP\left[\sup_{\theta\in\Theta}\left|\frac{1}{\hmax m_0}\sum_{j\in J'}f(\theta,z_{j}^{\theta_T(p)}) - \DPR(\theta_{T}(p),\theta)\right|> \frac{2\mathfrak{C}^*(f) + 2\sqrt{\log(T/\delta)}}{\sqrt{\hmax m_0}}\right].
\end{align*}
As before, we can argue that $\bbP[\neg \Eclean^{T,p}]\leq \delta/T$. Since $\ndistinct+(\pmax+1)\hmax \leq T$ by~Lemma~\ref{lemma:number}, it follows that $\bbP[\Eclean]\geq 1-\delta$ as $\ndistinct\leq T$.
\end{proof}

\subsection{Approximation Bounds under the data-driven setting}

In this section, we prove the following lemma analyzing the quality of the representative decision of each cell under the data-driven setting.

\begin{lemma}\label{lemma:cell-approximation}
Assume that $\Eclean$ holds for some $\delta\in(0,1)$. Let $\cP_{h,i}\in \cL_{h}$, and let $\cP_{h-1,j}$ be the parent cell of $\cP_{h,i}$. Then 
$$\PR(\theta_{h,i})\leq \inf_{\theta\in \cP_{h,i}}\PR(\theta)+2(2\sqrt{\Dtheta})^\alpha L_z\varepsilon2^{-\alpha h}+ \frac{4\mathfrak{C}^*(f) + 4\sqrt{\log(T/\delta)}}{\sqrt{\ndeploy_{h-1,j}m_0}}.$$
\end{lemma}
\begin{proof}
Let $\theta^\star_{h,i}\in \argmin_{\theta\in \cP_{h,i}}\PR(\theta)$.  By Lemma~\ref{lemma:PR-approximation},
$$\PR(\theta_{h,i}) \leq \DPR(\theta_{h-1,j},\theta_{h,i}) + (2\sqrt{\Dtheta})^\alpha L_z\varepsilon2^{-\alpha h}.$$
As $\Eclean$ holds, we have
$$\DPR(\theta_{h-1,j},\theta_{h,i})\leq \DPRhat(\theta_{h-1,j},\theta_{h,i}) + \frac{2\mathfrak{C}^*(f) + 2\sqrt{\log(T/\delta)}}{\sqrt{\ndeploy_{h-1,j}m_0}}.$$
Moreover,
$$\DPRhat(\theta_{h-1,j},\theta_{h,i})\leq \DPRhat(\theta_{h-1,j},\theta^\star_{h,i})\leq \DPR(\theta_{h-1,j},\theta^\star_{h,i})+ \frac{2\mathfrak{C}^*(f) + 2\sqrt{\log(T/\delta)}}{\sqrt{\ndeploy_{h-1,j}m_0}}$$
where the first inequality is due to our choice of $\theta_{h,i}$ minimizing $\DPRhat(\theta_{h-1,j},\theta)$ over $\theta\in\cP_{h,i}$ and the second inequality holds because $\Eclean$ holds. Lastly, by 
Lemma~\ref{lemma:PR-approximation},
$$ \DPR(\theta_{h-1,j},\theta^\star_{h,i})\leq \PR(\theta_{h,i}^\star) + (2\sqrt{\Dtheta})^\alpha L_z\varepsilon2^{-\alpha h}.$$
Consequently, it follows that
$$\PR(\theta_{h,i})\leq \PR(\theta^\star_{h,i})+ 2(2\sqrt{\Dtheta})^\alpha L_z\varepsilon2^{-\alpha h}+\frac{4\mathfrak{C}^*(f) + 4\sqrt{\log(T/\delta)}}{\sqrt{\ndeploy_{h-1,j}m_0}},$$
as required.
\end{proof}

\subsection{Proof of \Cref{thm:data-driven}}

Recall that $\bot_{h,p}$ is defined as the depth of the deepest cell containing the performative optimal solution~$\thetaPO$ opened for at least $2^p$ times until \Cref{algorithm1} finishes opening cells of depth $h$.

\begin{lemma}\label{lemma:depth}
Assume that the clean event $\Eclean$ holds for some $\delta\in(0,1)$, and let $d$ denote the $((2\sqrt{\Dtheta})^\alpha L_z\varepsilon, 2^{-\alpha},1)$-near-optimality dimension $d((2\sqrt{\Dtheta})^\alpha L_z\varepsilon,2^{-\alpha},1)$. For any $h\in[\lfloor\hmax/2^p\rfloor]$ and $p\in[0:\lfloor \log_2(\hmax/h)\rfloor]$, if the following condition holds, then $\bot_{h,p}=h$ with $\bot_{0,p}=0$.
$$\frac{2\mathfrak{C}^*(f) + 2\sqrt{\log(T/\delta)}}{\sqrt{2^pm_0}}\leq (2\sqrt{\Dtheta})^\alpha L_z\varepsilon 2^{-\alpha h}\quad \text{and}\quad \frac{\hmax}{h 2^p}\geq 2^{\alpha hd}.$$
\end{lemma}
\begin{proof}
Let $(h,p)$ with $h\in[\lfloor\hmax/2^p\rfloor]$ and $p\in[0:\lfloor \log_2(\hmax/h)\rfloor]$ satisfy the condition of the lemma. Then we will argue by induction that $\bot_{h',p}=h'$ for all $h'\in [h]$, thereby proving that $\bot_{h,p}=h$. 

Note that $\cP_{0,1}=\Theta$ contains $\thetaPO$ and $\cP_{0,1}$ is opened $\hmax$ times with $\hmax\geq 2^{\pmax}$, so $\bot_{0,p}=0$. Next, we assume that $\bot_{h'-1,p}=h'-1$ for some $h'\in[h]$. Then it is sufficient to show that $\bot_{h',p}=h'$. Let $i^\star_{h'-1}$ denote the index such that $\cP_{h'-1,i^\star_{h'-1}}$ is the cell containing $\thetaPO$ at depth $h'-1$. By the induction hypothesis, cell $\cP_{h'-1,i^\star_{h'-1}}$ is opened at least $2^p$ times, i.e., $\nopen_{h'-1,i^\star_{h'-1}}\geq 2^p$. This implies that $\nopen_{h'-1,i^\star_{h'-1}}\geq 2^p$ because $\ndeploy_{h'-1,i^\star_{h'-1}}\geq 2^{p'}=\nopen_{h'-1,i^\star_{h'-1}}$ for some $p'$ according to the design of \Cref{algorithm1}. Let $i^\star_{h'}$ denote the index such that $\cP_{h',i^\star_{h'}}$ is the cell containing $\thetaPO$ at depth $h'$. This means that $\cP_{h',i^\star_{h'}}$ is a child cell of $\cP_{h'-1,i^\star_{h'-1}}$ and $\ndeploy_{h',i^\star_{h'}}=\nopen_{h'-1,i^\star_{h'-1}}\geq 2^p$.

We open $\lfloor \hmax/h'2^p\rfloor$ cells from depth $h'$ cells with at least $2^p$ deployments. Suppose for a contradiction that cell $\cP_{h',i^\star_{h'}}$ is not one of them. This implies that for each solution $\theta_{h',i}$ of the $\lfloor \hmax/h'2^p\rfloor$ cells with $2^p$ deployments from depth $h'$, we have $\PRhat(\theta_{h',i})\leq \PRhat(\theta_{h',i^\star_{h'}}).$
Moreover,  such $\theta_{h',i}$ satisfies the following.
\begin{align}\label{eq:lemma-1}
\begin{aligned}
\PR(\theta_{h',i})-(2\sqrt{\Dtheta})^\alpha L_z\varepsilon2^{-\alpha h'}&\leq \PR(\theta_{h',i})-(2\sqrt{\Dtheta})^\alpha L_z\varepsilon2^{-\alpha h}\\
&\leq \PR(\theta_{h',i})-\frac{2\mathfrak{C}^*(f) + 2\sqrt{\log(T/\delta)}}{\sqrt{2^pm_0}}
\end{aligned}
\end{align}
where the first inequality holds because $h'\leq h$ and the second inequality holds due to the condition of the lemma. Furthermore,
\begin{align}\label{eq:lemma-2}
\begin{aligned}
\PR(\theta_{h',i})-\frac{2\mathfrak{C}^*(f) + 2\sqrt{\log(T/\delta)}}{\sqrt{2^pm_0}}&\leq \PR(\theta_{h',i})-\frac{2\mathfrak{C}^*(f) + 2\sqrt{\log(T/\delta)}}{\sqrt{\ndeploy_{h',i}m_0}}\\
&\leq \PRhat(\theta_{h',i})\\
&\leq \PRhat(\theta_{h',i^\star_{h'}})
\end{aligned}
\end{align}
where the first inequality holds because $\ndeploy_{h',i}\geq 2^p$ and the second inequality holds due to the assumption that $\Eclean$ holds. Combining~\eqref{eq:lemma-1} and~\eqref{eq:lemma-2}, we deduce that
$$\PR(\theta_{h',i})-(2\sqrt{\Dtheta})^\alpha L_z\varepsilon2^{-\alpha h'}\leq\PRhat(\theta_{h',i^\star_{h'}}).$$
Similarly, we can argue that
$$\PR(\theta_{h',i^\star_{h'}})+(2\sqrt{\Dtheta})^\alpha L_z\varepsilon2^{-\alpha h'}\geq\PRhat(\theta_{h',i^\star_{h'}}).$$
Consequently, it follows that
\begin{align*}
\PR(\theta_{h',i})&\leq \PR(\theta_{h',i^\star_{h'}})+2 (2\sqrt{\Dtheta})^\alpha L_z\varepsilon2^{-\alpha h'}\\
&\leq \inf_{\theta\in\Theta}\PR(\theta)+4(2\sqrt{\Dtheta})^\alpha L_z\varepsilon2^{-\alpha h'}+ \frac{4\mathfrak{C}^*(f) + 4\sqrt{\log(T/\delta)}}{\sqrt{2^pm_0}}
\end{align*}
where the second inequality follows from~Lemma~\ref{lemma:cell-approximation}, $\ndeploy_{h'-1,i^\star_{h'-1}}\geq 2^p$, and $\thetaPO$ is contained in cell $\cP_{h',i^\star_{h'}}$. Furthermore, by the condition of this lemma, it follows that
$$\PR(\theta_{h',i})\leq \PR(\thetaPO)+6(2\sqrt{\Dtheta})^\alpha L_z\varepsilon2^{-\alpha h'}.$$
In addition, since  $\thetaPO$ is contained in cell $\cP_{h',i^\star_{h'}}$, Lemma~\ref{lemma:cell-approximation} implies that
$$\PR(\theta_{h',i^*})\leq \PR(\thetaPO)+4(2\sqrt{\Dtheta})^\alpha L_z\varepsilon2^{-\alpha h'}.$$
This implies that
$$\cN_h(6(2\sqrt{\Dtheta})^\alpha L_z\varepsilon2^{-\alpha h'})\geq \left\lfloor\frac{\hmax}{h'2^p}\right\rfloor+1\geq \left\lfloor\frac{\hmax}{h2^p}\right\rfloor+1\geq 2^{\alpha hd}+1\geq 2^{\alpha h'd}+1$$
where $\lfloor{\hmax}/{h'2^p}\rfloor$ comes from cells $\cP_{h',i}$ and 1 is due to cell $\cP_{h',i^\star_{h'}}$ in the first inequality, the second and the fourth inequalities hold because $h'\leq h$, and the third ineuality comes from  the condition of the lemma. This in turn implies that $\cN_h(6(2\sqrt{\Dtheta})^\alpha L_z\varepsilon2^{-\alpha h'})>2^{\alpha h'd}$, a contradiction. Therefore, it follows that $\bot_{h',p}=h'$. Then the induction argument shows that $\bot_{h,p}=h$, as required.
\end{proof}

Next, we prove Lemma~\ref{lemma:regret-bound} which shows that
\begin{align*}
&\PR(\theta_T)- \PR(\thetaPO)\\
&\leq 2(2\sqrt{\Dtheta})^\alpha L_z\varepsilon2^{-\alpha (\bot_{\hmax,p}+1)}+\frac{8\mathfrak{C}^*(f) + 8\sqrt{\log(T/\delta)}}{\sqrt{2^pm_0}} + \frac{4\mathfrak{C}^*(f) + 4\sqrt{\log(T/\delta)}}{\sqrt{\hmax m_0}}.
\end{align*}
\begin{proof}[{\bfseries Proof of Lemma \ref{lemma:regret-bound}}]
Let $p\in[0:\pmax]$, and let 
$$(h,i)\in \argmin\limits_{(h,i)}\left\{\PRhat(\theta_{h,i}):h\in[\hmax+1],\cP_{h,i}\in\cL_h,\ndeploy_{h,i}\geq 2^p\right\}.$$
Recall that $\theta_T(p)$ is set to $\theta_{h,i}$ and that we obtain $\hmax m_0$ new samples from $\cD(\theta_T(p))$  from which we construct $\PRhat(\theta_T(p))$. Moreover, $\PRhat(\theta_T)\leq \PRhat(\theta_T(p))$. As $\Eclean$ holds, it follows that
\begin{align}\label{eq:thm1-1}
\begin{aligned}
\PR(\theta_T)-\frac{2\mathfrak{C}^*(f) + 2\sqrt{\log(T/\delta)}}{\sqrt{\hmax m_0}}&\leq \PRhat(\theta_T) \\&\leq \PRhat(\theta_T(p))\\&\leq \PR(\theta_T(p))+\frac{2\mathfrak{C}^*(f) + 2\sqrt{\log(T/\delta)}}{\sqrt{\hmax m_0}}.
\end{aligned}
\end{align}
Again, as $\Eclean$ holds and $\theta_T(p)=\theta_{h,i}$,
\begin{equation}\label{eq:thm1-2}\PR(\theta_T(p))\leq \PRhat(\theta_{h,i}) + \frac{2\mathfrak{C}^*(f) + 2\sqrt{\log(T/\delta)}}{\sqrt{\ndeploy_{h,i} m_0}}\leq\PRhat(\theta_{h,i}) + \frac{2\mathfrak{C}^*(f) + 2\sqrt{\log(T/\delta)}}{\sqrt{2^p m_0}}. 
\end{equation}
Recall that $\bot_{\hmax,p}$ is the depth of the deepest cell containing $\thetaPO$ opened for at least $2^p$ times until \Cref{algorithm1} finishes opening cells of depth $\hmax$.
Let $(\bot_{\hmax,p}+1, i^\star)$ denote the deepest cell containing $\thetaPO$ and a solution deployed at least $2^p$ times. By the choice of $(h,i)$, we have 
\begin{equation}\label{eq:thm1-3}\PRhat(\theta_{h,i})\leq \PRhat(\theta_{\bot_{\hmax,p}+1, i^\star})\leq \PR(\theta_{\bot_{\hmax,p}+1, i^\star})+\frac{2\mathfrak{C}^*(f) + 2\sqrt{\log(T/\delta)}}{\sqrt{2^pm_0}} 
\end{equation}
where the second inequality holds because $\Eclean$ holds and $\ndeploy_{\bot_{\hmax,p}+1, i^\star}\geq 2^p$. Moreover, since the parent cell of $\cP_{\bot_{\hmax,p}+1, i^\star}$ is opened at least $2^p$ times, it means that the parent cell contains a solution deployed at least $2^p$ times. Then by~Lemma~\ref{lemma:cell-approximation}, it follows that
\begin{equation}\label{eq:thm1-4}\PR(\theta_{\bot_{\hmax,p}+1, i^\star})\leq \PR(\thetaPO)+2(2\sqrt{\Dtheta})^\alpha L_z\varepsilon2^{-\alpha (\bot_{\hmax,p}+1)}+ \frac{4\mathfrak{C}^*(f) + 4\sqrt{\log(T/\delta)}}{\sqrt{2^p m_0}}.
\end{equation}
In summary, we deduce from~\eqref{eq:thm1-1} -- \eqref{eq:thm1-4} that
\begin{align*}
&\PR(\theta_T)- \PR(\thetaPO)\\
&\leq 2(2\sqrt{\Dtheta})^\alpha L_z\varepsilon2^{-\alpha (\bot_{\hmax,p}+1)}+\frac{8\mathfrak{C}^*(f) + 8\sqrt{\log(T/\delta)}}{\sqrt{2^pm_0}} + \frac{4\mathfrak{C}^*(f) + 4\sqrt{\log(T/\delta)}}{\sqrt{\hmax m_0}},
\end{align*}
as required.
\end{proof}

Note that the regret bound given by~Lemma~\ref{lemma:regret-bound} holds for any $p\in[0:\pmax]$.
Hence, to prove an upper bound on the regret $\PR(\theta_T)- \PR(\thetaPO)$, we need to choose an appropriate $p$ that achieves a small value of
$$2(2\sqrt{\Dtheta})^\alpha L_z\varepsilon2^{-\alpha (\bot_{\hmax,p}+1)}+\frac{8\mathfrak{C}^*(f) + 8\sqrt{\log(T/\delta)}}{\sqrt{2^pm_0}}.$$
As in~\cite{bartlett19a}, the strategy is to choose $p$ under which there is a strong lower bound on $\bot_{\hmax,p}+1$. In our case, however, we have the additional term $\tilde O(1/\sqrt{2^p})$. In fact, we will argue that the choice of $p$, under which $\bot_{\hmax,p}+1$ is large, also makes the additional term small. 

For simplicity, we use notations $\rho$, $\nu$, and $B$ defined as
$$\rho = 2^{-\alpha},\qquad \nu = (2\sqrt{\Dtheta})^\alpha L_z\varepsilon,\qquad B=\frac{2\sqrt{2}\left(\mathfrak{C}^*(f) + \sqrt{\log(T/\delta)}\right)}{\sqrt{m_0}}.$$
With these notations, Lemma~\ref{lemma:depth} can be restated as follows.
\begin{lemma}\label{lemma:depth'}
Assume that $\Eclean$ holds for some $\delta\in(0,1)$.
Let $d$ denote the $(\nu, \rho,1)$-near-optimality dimension $d(\nu,\rho,1)$. For any $h\in[\lfloor\hmax/2^p\rfloor]$ and $p\in[0:\lfloor \log_2(\hmax/h)\rfloor]$, if the following condition holds, then $\bot_{h,p}=h$.
$$\frac{B}{\sqrt{2^{p+1}}}\leq \nu \rho^{h}\quad \text{and}\quad \frac{\hmax}{h 2^p}\geq \rho^{-dh}.$$
\end{lemma}

Next, we define $\tilde h$ and $\tilde p$ as the numbers satisfying the following condition.
$$\frac{\hmax \nu^2\rho^{2\tilde h}}{\tilde h B^2}=\rho^{- \tilde h d}\quad \text{and}\quad \frac{B}{\sqrt{2^{\tilde p}}}=\nu \rho^{\tilde h}.$$
Then by definition of the Lambert $W$ function, we have
$$\tilde h = \frac{1}{(d+2)\log (1/\rho)}W\left(\frac{\hmax \nu^2(d+2)\log(1/\rho) }{B^2}\right).$$
Here, $B\geq L_z\varepsilon\cdot 2^{-\alpha \tilde h}$ holds if and only if $2^{\tilde p}\geq 1$. Hence, the case when $2^{\tilde p}\geq 1$ corresponds to the high-noise regime and the setting where
$2^{\tilde p}< 1$ corresponds to the low-noise regime. Next, as in~\cite{bartlett19a}, we define $\ddot{h}$ and $\ddot{p}$ as follows.
\begin{itemize}
    \item (High-noise regime) Set $\ddot h =\tilde h$ and $\ddot p = \tilde p$.
    \item (Low-noise regime) Set $\ddot h$ as $\ddot h =\bar h$ that satisfies $\hmax / \bar h = \rho^{-d \bar h}$ and $\ddot p=0$.
\end{itemize}
Note that for this choice of $\ddot h$ and $\ddot p$, we have
${\hmax}/{\ddot h 2^{\ddot p}}= \rho^{- d \ddot h}$.
under both regimes. Moreover, with~Lemma~\ref{lemma:depth'}, we may argue that the following statement holds.

\begin{lemma}[\citet{bartlett19a}]\label{lemma:bartlett}
Assume that $\Eclean$ holds for some $\delta\in(0,1)$.
Let $d$ denote the $(\nu, \rho,1)$-near-optimality dimension $d(\nu,\rho,1)$. Then $\ddot h \leq \tilde h$ and 
$$\bot_{\hmax,\lfloor\ddot p\rfloor}+1\geq \lfloor\ddot h\rfloor +1 \geq \ddot h$$
under both the high-noise and low-noise regimes.
\end{lemma}

Now we are ready to complete the proof of \Cref{thm:data-driven}.

\begin{proof}[{\bfseries Proof of \Cref{thm:data-driven}}]
Under~$\Eclean$, Lemmas~\ref{lemma:regret-bound} and~\ref{lemma:bartlett} imply that 
$$\PR(\theta_T)- \PR(\thetaPO)\leq 2(2\sqrt{\Dtheta})^\alpha L_z\varepsilon2^{-\alpha \ddot h}+\frac{8\mathfrak{C}^*(f) + 8\sqrt{\log(T/\delta)}}{\sqrt{2^{\lfloor \ddot p\rfloor}m_0}} + \frac{4\mathfrak{C}^*(f) + 4\sqrt{\log(T/\delta)}}{\sqrt{\hmax m_0}}$$
holds.

Let us first consider the low-noise regime. Since $2^{\tilde p}<1$, we know that $B<\nu\rho^{\tilde h}$. By Lemma~\ref{lemma:bartlett}, we have $\ddot h\leq \tilde h$, which implies that $B<\nu\rho^{\tilde h}\leq \nu\rho^{\ddot h}$. Then as $\ddot p=0$ under the low-noise regime, it follows that 
$$\PR(\theta_T)- \PR(\thetaPO)\leq (2+2\sqrt{2})(2\sqrt{\Dtheta})^\alpha L_z\varepsilon2^{-\alpha \ddot h}+ \frac{4\mathfrak{C}^*(f) + 4\sqrt{\log(T/\delta)}}{\sqrt{\hmax m_0}}.$$
Moreover, if $\hmax\geq 1$, then as $B\leq \nu \rho^{\ddot h}$, 
$$\PR(\theta_T)- \PR(\thetaPO)\leq (2+3\sqrt{2})(2\sqrt{\Dtheta})^\alpha L_z\varepsilon2^{-\alpha \ddot h}.$$
When $d=0$, we have $\ddot h = \hmax$. When  $d>0$, we have
$$\ddot h = \frac{1}{\alpha d\log 2}W\left(\hmax \alpha d\log 2\right).$$
For the high-noise regime, we have $\ddot h = \tilde h$ and
$$\frac{8\mathfrak{C}^*(f) + 8\sqrt{\log(T/\delta)}}{\sqrt{2^{\lfloor \ddot p\rfloor}m_0}}=\frac{4 B}{\sqrt{2^{\lfloor \ddot p\rfloor+1}}}\leq \frac{4 B}{\sqrt{2^{\ddot p}}}=4\nu \rho^{\tilde h}.$$
Therefore, under the high-noise regime, we have
$$\PR(\theta_T)- \PR(\thetaPO)\leq 6(2\sqrt{\Dtheta})^\alpha L_z\varepsilon2^{-\alpha \tilde h}+ \frac{4\mathfrak{C}^*(f) + 4\sqrt{\log(T/\delta)}}{\sqrt{\hmax m_0}}.$$
Recall that $\tilde h$ is given by 
$$\tilde h = \frac{1}{\alpha(d+2)\log 2}W\left(\frac{ (4\Dtheta)^\alpha \alpha(d+2)\log 2 }{8(\mathfrak{C}^*(f) + 4\sqrt{\log(T/\delta)})^2}L_z^2\varepsilon^2m_0\hmax\right).$$
Lastly, \cite{Hoorfar2008} showed that if $x\geq e$, then $W(x)\geq \log(x/\log (x))$. Hence, if $d>0$ and $\hmax \alpha d \log 2\geq e$ under the low-noise regime, then $\theta$ satisfies
$$\PR(\theta) - \PR(\thetaPO) \leq (2+3\sqrt{2})(2\sqrt{\Dtheta})^\alpha L_z\varepsilon\left(\frac{\hmax \alpha d \log 2}{\log(\hmax \alpha d \log 2)}\right)^{-\frac{1}{d}}.$$
Moreover, if $B^2 \hmax \nu^2\alpha (d+2)\log2\geq e$, then
\begin{align*}
&\PR(\theta) - \PR(\thetaPO) \\
&\leq 6(2\sqrt{\Dtheta})^\alpha L_z\varepsilon\left(\frac{\hmax \nu^2\alpha (d+2)\log2 /B^2}{\log(\hmax \nu^2\alpha (d+2)\log2 /B^2)}\right)^{-\frac{1}{d+2}}+\frac{4\mathfrak{C}^*(f) + 4\sqrt{\log(T/\delta)}}{\sqrt{\hmax m_0}},
\end{align*}
as required.
\end{proof}

\vskip 0.2in
\bibliography{mybibfile}

\end{document}